\documentclass{svjour3draft}                     %
\smartqed  %
\usepackage{graphicx}

\journalname{jmiv}
\usepackage{amsfonts,amsmath,amssymb}
\usepackage{mathtools}
\usepackage{algorithm}
\usepackage{algorithmic}
\usepackage[textwidth=4cm,textsize=footnotesize]{todonotes}
\usepackage{xargs}
\usepackage[Symbolsmallscale]{upgreek}

\newtheorem{assumption}{\textbf{H}\hspace{-3pt}}

\def\vx{x}

\def\Ly{\Lt_y}

\def\MAP{\mathrm{MAP}}

\def\dim{d}

\def\dimY{m} %

\newcommand{\cball}[2]{\overline{\operatorname{B}}(#1,#2)}

\newcommand{\prox}{\operatorname{prox}}

\newcommand{\Lt}{\mathtt{L}}

\newcommand{\bvareps}{\vareps_0}

\def\x{{x}}

\newcommandx{\norm}[2][1=]{\ifthenelse{\equal{#1}{}}{\left\Vert #2 \right\Vert}{\left\Vert #2 \right\Vert^{#1}}}
\newcommandx{\normLigne}[2][1=]{\ifthenelse{\equal{#1}{}}{\Vert #2 \Vert}{\Vert #2\Vert^{#1}}}

\def\msa{\mathsf{A}}
\def\mst{\mathsf{T}}

\def\msk{\mathsf{K}}
\def\mss{\mathsf{S}}

\def\msb{\mathsf{B}} 
\def\msc{\mathsf{C}}
\def\mse{\mathsf{E}}

\def\msu{\mathsf{U}}
\def\msv{\mathsf{V}}

\newcommand{\mcb}[1]{\mathcal{B}(#1)}

\def\mcf{\mathcal{F}}

\def\rset{\mathbb{R}}

\def\cset{\mathbb{C}}

\def\nset{\mathbb{N}}
\def\nsets{\mathbb{N}^*}

\def\rmi{\mathrm{i}}

\def\rmd{\mathrm{d}}

\def\rmc{\mathrm{C}}

\def\rmQ{\mathrm{Q}}
\def\rmP{\mathrm{P}}

\def\rmA{\mathrm{A}}

\newcommand{\R}{\mathbb R}

\newcommand{\argmax}{\operatorname*{arg\,max}}
\newcommand{\argmin}{\operatorname*{arg\,min}}

\newcommand{\1}{\mathbbm{1}}

\newcommand{\LeftEqNo}{\let\veqno\@@leqno}

\newcommand{\PE}{\mathbb{E}}

\newcommand{\abs}[1]{\left\vert #1 \right\vert}
\newcommand{\absLigne}[1]{\vert #1 \vert}

\newcommand{\defEns}[1]{\left\lbrace #1 \right\rbrace }

\newcommand{\expeLigne}[1]{\PE [ #1 ]}

\def\eqsp{\;}

\newcommand{\ocint}[1]{\left(#1\right]}
\newcommand{\ooint}[1]{\left(#1\right)}

\newcommand{\ocintLigne}[1]{(#1]}

\newcommand{\ball}[2]{\operatorname{B}(#1,#2)}

\newcommand{\opnorm}[1]{{\left\vert\kern-0.25ex\left\vert\kern-0.25ex\left\vert #1 
    \right\vert\kern-0.25ex\right\vert\kern-0.25ex\right\vert}}

\def\Id{\operatorname{Id}}

\newcommand{\CPP}[3][]
{\ifthenelse{\equal{#1}{}}{{\mathbb P}\left(\left. #2 \, \right| #3 \right)}{{\mathbb P}_{#1}\left(\left. #2 \, \right | #3 \right)}}

\def\Id{\operatorname{Id}}

\newcommand{\ensemble}[2]{\left\{#1\,:\eqsp #2\right\}}
\newcommand{\ensembleLigne}[2]{\{#1\,:\eqsp #2\}}

\newcommand\coupling[2]{\Gamma(\mu,\nu)}

\newcommand{\complementary}{\mathrm{c}}

\renewcommand{\geq}{\geqslant}
\renewcommand{\leq}{\leqslant}

\def\vareps{\varepsilon}

\newcommand \ie {{\em i.e. }}

\def\pnpsgd{PnP-SGD}
\def\Mtt{\mathtt{M}}

\def\Ltt{\mathtt{L}}

\usepackage{amsmath,amsfonts,bm}

\def\1{\bm{1}}

\def\Ltt{\mathtt{L}}

\def\vx{x}

\def\eqsp{\;}
\def\Id{\operatorname{Id}}

\def\rmA{{\mathbf{A}}}

\def\rmP{{\mathbf{P}}}
\def\rmQ{{\mathbf{Q}}}

\def\vu{{\bm{u}}}
\def\vv{{\bm{v}}}

\def\vx{{{x}}}  
\def\vy{{{y}}}
\def\vz{{{z}}}
\def\vv{{{v}}}

\DeclareMathAlphabet{\mathsfit}{\encodingdefault}{\sfdefault}{m}{sl}
\SetMathAlphabet{\mathsfit}{bold}{\encodingdefault}{\sfdefault}{bx}{n}

\def\vareps{\varepsilon}
\def\rset{\mathbb{R}}

\newcommand{\scalar}[2]{{\left\langle\,#1,\,#2\right\rangle}} %

\newcommand{\Fdata}{\ensuremath{F}}
\newcommand{\Gprior}{\ensuremath{U}}

\usepackage{hyperref}

\usepackage{rotating} %

\newcommand{\CenteredVcell}[2]{\begin{centering}\begin{sideways}\parbox[c]{#1}{\begin{centering}{\hfill #2 \hfill\textcolor{white}{.}}\end{centering}}\end{sideways}\end{centering}}

\newcommand{\nburnin}{\ensuremath{n_{\mathtt{burnin}}}}
\newcommand{\deltaStable}{\ensuremath{\delta_{\mathtt{stable}}}}

\begin{document}

\title{On Maximum-a-Posteriori estimation with Plug \& Play priors and stochastic gradient descent}

\author{
  Rémi Laumont~\footnotemark[2]~\footnotemark[3]%
    \footnotetext[2]{These authors contributed equally}%
    \footnotetext[3]{Universit\'e de Paris, CNRS, MAP5 UMR 8145, F-75006 Paris, France}%
  \and Valentin De Bortoli~\footnotemark[2]~\footnotemark[4]%
    \footnotetext[4]{Department of Statistics
 University of Oxford
 24-29 St Giles
 OX1 3LB, Oxford
 United Kingdom}%
  \and Andr\'{e}s Almansa~\footnotemark[3]
  \and Julie Delon~\footnotemark[3]~\footnotemark[7]%
  \footnotetext[7]{Institut Universitaire de France (IUF)}
  \and Alain Durmus~\footnotemark[5]%
    \footnotetext[5]{Centre Borelli, UMR 9010, 
    \'Ecole Normale Supérieure Paris-Saclay}%
  \and Marcelo Pereyra~\footnotemark[6]%
    \footnotetext[6]{School of Mathematical and Computer Sciences, Heriot-Watt University \& Maxwell Institute for Mathematical Sciences, Edinburgh, United Kingdom}
  }
\author{
  Rémi Laumont
  \and Valentin De Bortoli \footnote{Corresponding author}
  \and Andr\'{e}s Almansa
  \and Julie Delon
  \and Alain Durmus
  \and Marcelo Pereyra~
  }

\institute{R. Laumont, A. Almansa and J. Delon \at
              Universit\'e de Paris, CNRS, MAP5 UMR 8145, F-75006 Paris, France
           \and
           V. De Bortoli \at
              Department of Statistics,
 University of Oxford,
 24-29 St Giles
 OX1 3LB, Oxford
 U.K. \\ \email{valentin.debortoli@gmail.com}
 \and
           A. Durmus  \at Centre Borelli, UMR 9010, 
    \'Ecole Normale Supérieure Paris-Saclay, France.
\and Marcelo Pereyra \at  Heriot-Watt University \& Maxwell Institute for Mathematical Sciences, Edinburgh, U.K.
}

\date{Received: date / Accepted: date}

\maketitle

\begin{abstract}
Bayesian methods to solve imaging inverse problems usually combine an explicit data likelihood function with a prior distribution that explicitly models expected properties of the solution. Many kinds of priors have been explored in the literature, from simple ones expressing local properties to more involved ones exploiting image redundancy at a non-local scale. In a departure from explicit modelling, several recent works have proposed and studied the use of implicit priors defined by an image denoising algorithm. This approach, commonly known as Plug \& Play (PnP) regularisation, can deliver remarkably accurate results, particularly when combined with state-of-the-art denoisers based on convolutional neural networks. However, the theoretical analysis of PnP Bayesian models and algorithms is difficult and works on the topic often rely on unrealistic assumptions on the properties of the image denoiser. This papers studies maximum-a-posteriori (MAP) estimation for Bayesian models with PnP priors. We first consider questions related to existence, stability and well-posedness, and then present a convergence proof for MAP computation by PnP stochastic gradient descent (PnP-SGD) under realistic assumptions on the denoiser used. We report a range of imaging experiments demonstrating PnP-SGD as well as comparisons with other PnP schemes.        

\end{abstract}

\section{Introduction}
\label{intro}
Many inverse problems in imaging sciences consider the estimation of an unknown image $\vx \in \mathbb{R}^d$ from an observation $\vy$, related to $\vx$ as follows
\begin{equation}\vy = \rmA(\vx) + n \eqsp , 
\label{eq:forward}
  \end{equation} 
where $\rmA$ is an observation operator and $n$ is additive noise. Equation~\eqref{eq:forward} is commonly referred to as the forward model.  

Recovering $\vx$ from the observation $\vy$ by inverting the observation model is usually an ill-posed or ill-conditioned problem, in the sense that the the solution is not unique, or it is not stable w.r.t. perturbations of the observation $y$. To reduce estimation uncertainty and provide meaningful solutions it is necessary to use additional information about the unknown image $\vx$ so that the estimation problem becomes well-posed\footnote{A problem is said to be well-posed in the sense of Hadamard when a solution exists, is unique, and depends in a Lipschitz continuous manner w.r.t. the observed data $y$.}. 

The Bayesian statistical framework provides a powerful framework to formulate well-posed solutions to such imaging inverse problems. In this framework, the likelihood of the observation $\vy$ given the unknown $\vx$ is described by a statistical model with probability density function $p(\vy|\vx)$, and assumptions on the unknown $\vx$ take the form of a marginal or prior density $p(\vx)$. These densities are usually specified explicitly, either directly or via their potentials $\Gprior(\vx)= -\log p(\vx)$ and $\Fdata(\vx,\vy) = -\log p(\vy|\vx)$. Observed and prior information are then combined by using Bayes' theorem to derive the posterior distribution of $\vx$ given $\vy$, with probability density function given by
 \begin{equation}
p(\vx|\vy) = \frac{p(\vy|\vx)p(\vx)}{\int p(\vy|\tilde{\vx})p(\tilde{\vx})\textrm{d}\tilde{\vx}} \eqsp .
\label{eq:posterior}
\end{equation}
This model underpins our inferences about $\vx|\vy$ and provides the basis for deriving Bayesian estimates. While different Bayesian estimators can then be considered, the Bayesian imaging literature predominantly relies on the maximum-a-posteriori (MAP) estimator
\begin{eqnarray}
\label{eq:MAP}
\hat{\vx}_{\textsc{map}} &=& \argmax_{x \in \rset^d} p(\vx|\vy) \! =\! \argmin_{x \in \rset^d} \left \{\Fdata(\vx,\vy) + \Gprior(\vx)\right \} \eqsp , \end{eqnarray}
which is usually computationally cheaper than other estimators that require computing expectations w.r.t $x|y$, such as the the Minimum Mean Square Error estimator (MMSE) $\hat{\vx}_{\textsc{mmse}} =  \mathbb{E}[\vx|\vy] = \int_{\rset^d} \tilde{x} p(\tilde{x}|y)\textrm{d}\tilde{x}$.

Until recently, most approaches in Bayesian imaging relied on explicit priors such as Markov random fields \cite{MRF-MIT-2011} (with the fields based on the total-variation pseudo-norm and its approximations being particularly prominent examples \cite{Rudin1992,Chambolle04,Louchet2013}), priors expressing sparsity in a transformed domain~\cite{donoho1995noising}, or learning-based priors like patch-based Gaussian mixture models~\cite{Zoran2011,yu2011solving,Teodoro2018scene}. 
Among these priors, log-concave models have been particularly favored for both computational and analytical reasons. With regards to computation, log-concavity leads to formulations for MAP estimation and uncertainty quantification which benefit from the full arsenal of convex optimization tools, scaling efficiently to high-dimensions, with strong and well-known convergence guarantees \cite{chambolle2011first,parikh2014proximal,bubeck2014convex,bauschke2011convex,pereyra2017,repetti_pereyra_2019}. Similarly, log-concavity also enables the use of state-of-the-art Monte Carlo sampling algorithms (see, e.g., \cite{durmus2018efficient,pereyra2020accelerating}). Moreover, from an analytical viewpoint, log-concavity guarantees the well-posedness of $p(x|y)$, and that $\hat{\vx}_{\textsc{map}}$ is formally a Bayesian estimator (as opposed to simply being the point with greatest density w.r.t. the Lebesgue measure, which is a significantly weaker result, see \cite{pereyra2019b} for details).

\paragraph{Computation of the MAP solution.} When the posterior density $p(.|\vy)$ is proper and differentiable, with
$\nabla \log p(.|\vy)$ Lipschitz continuous, it is possible to use
first-order optimisation methods to compute maximisers of $p(.|\vy)$, \ie \ MAP
estimators. The simplest first order optimisation scheme to compute $\hat{\vx}_{\textsc{map}}$ is arguably the gradient descent algorithm, given by an initial state
$X_0 \in \rset^d$ and the following recursion for all $k \in \nset$
\begin{equation}\label{GD}
X_{k+1} = X_{k} - \delta_k \nabla \Fdata (X_k,\vy) - \delta_k \nabla \Gprior ({X}_k) \eqsp ,
\end{equation}
where $(\delta_k)_{k \in \nset} \in (\rset_+)^\nset$ is a sequence of
step-sizes.  The sequence $(X_k)_{k \in \nset}$ converges to critical points of
$p(.|y)$ under mild assumptions on the sequence 
$(\delta_k)_{k\in \nset}$ \cite{nesterov2018lectures} and
$p(.|\vy)$. Alternatively, the stochastic gradient descent (SGD) variant
\begin{equation}\label{SGD}
X_{k+1} = X_{k} - \delta_k \nabla \Fdata (X_k,\vy) - \delta_k \nabla \Gprior ({X}_k)+ {\delta_k} Z_{k+1} \eqsp , 
\end{equation}
where $\ensembleLigne{Z_k}{k \in \nset}$ is a family of i.i.d Gaussian random
variables with zero mean and identity covariance matrix, is more robust to local
minima and saddle points and hence more suitable when $x \mapsto p(x|y)$ is not
log-concave on $\mathbb{R}^d$
\cite{brandiere196algorithmes,bottou2018optimization}. Of course, there are many
optimisation schemes with better convergence properties than SGD (see, e.g.,
\cite{nesterov2018lectures} in the convex case), as well as other dynamics to
construct efficient optimisation algorithms
\cite{kingma2014adam,zou2019sufficient}. Nevertheless, SGD is straightforward to
apply, robust, and has a detailed convergence theory, making it a valuable algorithm in the imaging scientist's toolbox.

\paragraph{Deep learning approaches.}
In the last few years, deep neural networks have become ubiquitous to solve
inverse problems in imaging, showing unmatched performances for point estimation for some specific
problems like image denoising. Deep networks can be trained without explicitly using the
knowledge of the forward
model~\eqref{eq:forward}~\cite{dong2014learning,zhang2017beyond,zhang2018ffdnet,gharbi2016deep,schwartz2018deepisp,gao2019dynamic}
or on the contrary can use this model explicitly via unrolled optimization
techniques~\cite{gregor2010learning,Chen2017,diamond2017unrolled,gilton2019neumann}. One disadvantage of using neural networks to solve imaging inverse problems is that in order to achieve state-of-the-art performance it is usually necessary to train the network for a specific problem configuration - the network must be
retrained if the forward model or any model parameters change significantly. Also, imaging approaches based on neural networks struggle to support more advanced inferences, such as decision-theoretic procedures.

\paragraph{Plug \& Play (PnP) approaches.}
PnP approaches strike a balance between an explicit and fully modular modelling paradigm that represents the likelihood $p(y|x)$ and the prior $p(x)$ explicitly (or the data-fidelity and regularisation terms in a variational formulation), and a purely data-driven approach that seeks to infer the model from data. More precisely, these methods usually combine an explicit
likelihood density or data-fidelity term with a prior or regularisation term that is implicitly defined by an image denoising algorithm $D_\varepsilon$ \cite{arridge_maass_oktem_schonlieb_2019}. This construction takes place at the algorithmic level, as opposed to at an explicit modelling level, by using the denoiser $D_\varepsilon$ in lieu of the gradient $\nabla \Gprior$ (also called score in the literature) or the proximal operator $\prox_{\Gprior}$ within an iterative optimisation scheme to compute $\hat{\vx}_{\MAP}$
\cite{meinhardt2017learning,Zhang2017,chan2017plug,kamilov2017plug,ryu2019plug}. This strategy allows decoupling the data observation model from the regularisation (that can be represented by a neural network denoiser $D_\varepsilon$ learnt from training data), and has been shown to deliver remarkably accurate results for a large panel of inverse problems, particularly when $D_\varepsilon$ is chosen carefully.  The question of the convergence of these PnP algorithms has been the focus of several papers in the
recent years~\cite{ryu2019plug,Xu2020,sun2020scalable}, but it has not been
satisfactorily answered yet, especially with regards to the strong assumptions made on $D_{\vareps}$, on $\Fdata$, and on the algorithm parameters. All of these limitations will be detailed in Section~\ref{sec:pnpmethods}. 

Lastly, many other fundamental questions related to inference with PnP schemes remain largely unexplored, particularly in the context of the Bayesian paradigm. For example, questions related to the correct definition of the Bayesian models, to the existence and well-posedness of the estimators that the PnP scheme seeks to compute, and whether these are proper Bayesian estimators in the sense of decision theory.

\paragraph{Contributions.} The aim of this paper is to significantly improve our theoretical understanding of MAP estimation with \emph{Plug \& Play} priors. We first study some fundamental questions related to the posterior density that are essential for meaningful MAP estimation. We establish easily verifiable conditions such that the PnP posterior density is proper, well-posed, Lipschitz continously differentiable, and stable w.r.t. the parameter $\vareps$ defining the strengh of the denoiser $D_\vareps$. Following on from this, we investigate in detail the convergence of the Plug-and-Play Stochastic Gradient Descent (PnP-SGD) for MAP estimation.
This iterative  algorithm takes the following form: for $X_0 \in \rset^d$ and any $k \in \nset$
 \begin{equation}\label{eq:pnp-sgd}
 X_{k+1} = X_k - \delta_k \nabla \Fdata(X_k,\vy) - (\delta_k/\vareps) (X_k - D_{\vareps}(X_k)) + \delta_k Z_{k+1} 
\eqsp \tag{PnP-SGD},
 \end{equation}
 where $\ensembleLigne{Z_k}{k \in \nset}$ is a family of independent Gaussian
 random variables with zero mean and identity covariance matrix,
 $D_{\vareps}: \ \rset^d \to \rset^d$ is a denoiser operator and
 $(\delta_k)_{k \in \nset}$ is a sequence of step-sizes.  We establish convergence results for PnP-SGD under mild and realistic assumptions on
 $D_{\vareps}$ that hold for several well-known denoisers, including state-of-the-art denoisers based on convolutional neural networks.
 We also provide extensive experiments on several canonical inverse problems,
 implementing PnP-SGD with the denoising neural network presented
 in~\cite{ryu2019plug}, which satisfies our convergence guarantees. Using the
 same denoising network, we also implement other state-of-the-art Plug-and-Play
 methods, and show that all schemes provide close
 results (in terms of image quality) when they converge. Although PnP-SGD is
 often slower than schemes using $D_{\vareps}$ to approximate a proximal
 operator, the conditions of convergence provided in this paper for PnP-SGD are
 less restrictive than those provided in the literature for other PnP schemes,
 and convergence is possible for any regularization parameter balancing the
 weights of the data and prior terms.

 The paper is organized as follows. Section~\ref{sec:pnpmethods} presents an overview of previous works on \textit{Plug \& Play} approaches for MAP
 estimation. Following on from this, Section~\ref{sec:pnp-theory-alg} describes the proposed theoretical framework for analysing MAP estimation with PnP priors, as well as a detailed convergence theory for MAP computation by PnP-SGD.
 Section~\ref{sec:experiments} illustrates the behavior of PnP-SGD, along with
 other PnP schemes, on several classical imaging problems. %

\section{A survey of Plug \& Play methods in imaging}
\label{sec:pnpmethods}
In the context of imaging inverse problems, \emph{Plug \& Play} methods aim at
using a carefully chosen denoiser $D_{\vareps}: \ \rset^d \to \rset^d$ to
implicitly define an image prior. This is achieved by relating $D_{\vareps}$ to
a proximal operator or a gradient associated with the prior density.  In the
first case, $D_\varepsilon$ replaces a MAP estimator for a denoising problem. In
the second case, $D_\varepsilon$ replaces a Minimum Mean Square Error (MMSE)
estimator for a denoising problem, related to the gradient of a log-prior via
Tweedie's identity~\footnote{ Notice that although it is conceptually helpful to
  distinguish these two cases (in order to make a historical and practical
  survey of the subject), there are clear theoretical connections between the
  two approaches.  Indeed, under regularity conditions on the Bayesian model
  involved, MAP denoisers can be expressed as MMSE denoisers under an
  alternative (albeit often unknown) Bayesian model
  \cite{gribonval2011should}. %
However this equivalence can not always be exploited in practice and has been
mostly ignored in the literature on \emph{Plug \& Play} methods until very
recently with the work of Xu \emph{et al.} \cite{Xu2020} to be presented later.}\cite{efron2011tweedie}.

In what follows, we describe how these approaches have been widely used to
compute solutions to inverse problems. In our discussion, we pay particular
attention to questions related to algorithmic convergence, and to the
interpretation of the computed solutions, as this has been an important focus of
the literature. 

\subsection{{{Plug \& Play} MAP estimators using proximal splitting}}
Let $D^\dagger_\varepsilon$ denote the MAP estimator to recover $x$ from a noisy
observation $x_{\vareps} \sim \mathcal{N}(x, \vareps \Id)$ under the assumption
that $x$ has marginal density $p(\vx) \propto \exp[-\Gprior(\vx)]$; that is,
$D^\dagger_\varepsilon(x_{\vareps}) = \argmin_{{x} \in \rset^d} \{\frac12
\|x_{\vareps} - {\vx}\|^2 + \varepsilon \Gprior(\vx)\} = \prox_{\varepsilon
  \Gprior}(x_{\vareps})$. When we set the PnP denoiser $D_\vareps$ such that
$D_\varepsilon = D^\dagger_\varepsilon$, any optimization scheme making use of a
proximal descent on the prior can be used to solve~\eqref{eq:MAP} via
$D_\varepsilon$.

For instance, the alternating direction method of multipliers (ADMM)~\cite{boyd2011distributed} writes the augmented Lagrangian of~\eqref{eq:MAP} as
\[\textstyle{E_{\vareps}(\vx, \vz,\vv) = \Fdata(\vx,\vy) +  \|\vx-\vz\|^2/(2 \vareps)+ \vv^\top (\vx - \vz)+ \Gprior(\vz)\eqsp.}\]
The joint optimization of the augmented Lagrangian is given by 
\[(\hat{\vx}_{\textsc{map}}, \hat{\vz}_{\textsc{map}}) =
  \textstyle{\argmin_{\vx,\vz \in \rset^d}\max_{\vv \in \rset^d}}
  E_{\vareps}(\vx, \vz, \vv).\] This provides the solution
$\hat{\vx}_{\textsc{map}}= \hat{\vz}_{\textsc{map}}$ of~\eqref{eq:MAP} when
$\vareps \to 0$.
 In practice, the joint optimization is solved by an alternate minimization scheme on $\vx$ and $\vz$ and a gradient ascent on $\vu=\vareps\vv$,
 \begin{eqnarray}
 \vx_{k+1} & = \textstyle{\argmin_\vx E_\varepsilon(\vx,\vz_k,\vu_k/\vareps)} & = \prox_{\varepsilon \Fdata(\cdot,\vy)} (\vz_k - \vu_k) \eqsp , \label{eq:ADMM-data-step}
 \\
   \vz_{k+1} & = \textstyle{\argmin_\vz E_\varepsilon(\vx_{k+1},\vz,\vu_k/\vareps)} & = \prox_{\varepsilon \Gprior} (\vx_{k+1} + \vu_k) = D_\varepsilon(\vx_{k+1}+\vu_k) \eqsp ,  \label{eq:ADMM-denoising-step}\\
   \vu_{k+1} & = \vu_k +\vx_{k+1} - \vz_{k+1} \eqsp . 
 \end{eqnarray}
 Similarly, when $\Fdata(.,\vy) $ is differentiable, the simpler
 Forward-backward splitting (FBS) scheme~\cite{combettes2011proximal}, which
 only requires to compute $\nabla F$, can be written in a Plug-and-Play fashion
 as
  \begin{equation}
 \vx_{k+1}  = \prox_{\varepsilon \Gprior}(\vx_{k} - \vareps \nabla F (\vx_{k} ,\vy)) =  D_\varepsilon (\vx_{k} - \vareps \nabla F (\vx_{k} ,\vy)) \eqsp . \label{eq:FBSstep}
\end{equation}
A fully proximal version of this algorithm, called Backward-backward splitting (BBS)~\cite{combettes2011proximal}, writes 
 \begin{equation}
 \vx_{k+1}  = \prox_{\varepsilon \Gprior}( \prox_{\varepsilon \Fdata}(\vx_{k})) =  D_\varepsilon (\prox_{\varepsilon \Fdata}(\vx_{k})) \eqsp .\label{eq:BBSstep}
\end{equation}
BBS aims at solving a slightly modified version of~\eqref{eq:MAP} where $F$ is
replaced by its Moreau envelope with parameter $\varepsilon$. The same algorithm
can be derived using half-quadratic splitting to solve~\eqref{eq:MAP}.
  
When $ \Gprior$ is convex, such splitting schemes and many variants (including 
primal-dual methods, ISTA or FISTA, etc.) are well understood and proved to converge
to the global optimum~\cite{boyd2011distributed}. They have also been
successfully used for non-convex $\Gprior$ like patch-based Gaussian mixture
models (GMM) as pioneered for external learning by Zoran \& Weiss
in~\cite{Zoran2011}. The use of splitting schemes with non-convex GMM priors was
later refined with convergence guarantees for scene-adapted learning
\cite{teodoro2018convergent}. 

Following the seminal work \cite{venkatakrishnan2013plug}, this kind of
splitting schemes have become ubiquitous in cases where $\Gprior$ (and hence
$D^\dagger_\varepsilon$) are unknown and unspecified, but a denoiser
$D_\varepsilon$ is available and assumed to be a good approximation of
$D^\dagger_\varepsilon =
\prox_{\varepsilon\Gprior}$. %
As popular and efficient these methods have become, their convergence properties
have remained largely unknown. Indeed, for most denoisers $D_\varepsilon$, there
is no guarantee that there exists a potential $\Gprior$ such that
$D_\varepsilon = \prox_{\varepsilon\Gprior}$.
In~\cite{Sreehari2015}, Sreehari
\emph{et al.} establish some sufficient conditions for this to happen:
$D_\varepsilon$ must be differentiable, and its Jacobian $J_{D_\varepsilon}$
should be symmetric %
with eigenvalues within the $[0,1]$ interval to ensure
non expansiveness. 
These assumptions hold for transform-domain thresholding
denoisers and for variants of Non Local means \cite{buades2005non} where
symmetry is explicitly enforced~\cite{Sreehari2015}.  However, the two
assumptions are unfortunately false for most popular denoisers, including Non
Local Means~\cite{buades2005non}, BM3D~\cite{dabov2006image}, Non Local
Bayes~\cite{lebrun2013nonlocal} and neural networks denoisers like
DnCNN~\cite{zhang2017beyond}, as observed in
\cite{reehorst2018regularization}. %

\paragraph{Consensus equilibirum / fixed point interpretation.} Since it remains difficult to show that PnP schemes converge to the MAP or even a critical point of~\eqref{eq:MAP}, several authors have proposed to analyse these schemes from a consensus equilibrium point of view~\cite{buzzard2018plug,ahmad2020plug}, or similarly to consider and analyse these approches as fixed-point algorithms~\cite{sun2019online,ryu2019plug}. 
The fixed points attained by these algorithms cannot be interpreted as MAP estimators, but should be seen as solving a set of equilibrium equations involving both the denoiser and the data term. 
For instance, for PnP-FBS, the idea is to show convergence to the set of points
$\vx$ satisfying
$\vx = D_{\varepsilon} (\vx - \varepsilon \nabla \Fdata (\vx ,\vy))$. It can
easily be shown that the fixed points of several of these PnP algorithms (in
particular PnP-ADMM and PnP-FBS)
coincide~\cite{meinhardt2017learning,sun2019online}.

In~\cite{sun2019online}, assuming that such fixed points exist, Sun et al. show
convergence of PnP-ISTA (which is equivalent to PnP-FBS
above) %
under %
the assumptions that $\nabla F$ is $\Ly$-Lipschtitz,
$\varepsilon \Ly \leq 1$ and $D_{\varepsilon}$ is $\theta$-averaged, see
\cite[Definition 4.33]{bauschke2011convex} for a definition. This assumption on
the denoiser is probably too strong, since most denoisers cannot be considered
as averaged operators.  In~\cite{sun2020scalable}, Sun et al. reformulate
PnP-ADMM with different convergence conditions, and still assume quite
restrictive conditions on the denoiser
$D_\varepsilon$~\footnote{In~\cite{sun2020scalable}, the residual
  $\Id -D_\varepsilon $ is assumed to be firmly non expansive, which is
  equivalent to say that $D_\varepsilon$ is firmly non expansive, see
  \cite[Proposition 4.4]{bauschke2011convex}.}.

In~\cite{ryu2019plug}, Ryu \emph{et al.} propose a convergence analysis of
PnP-ADMM, PnP-FBS and PnP-DRS (PnP Douglas-Rachford Splitting), based on the
weaker assumption that the residual operator $D_\varepsilon - \Id$
is $\Ltt$-Lipschitz 
with a Lipschitz constant which depends both on the data fitting term $\Fdata$ and the denoiser $D_{\varepsilon}$. 
The proof also requires $\Fdata$ to be
$\mu$-strongly convex (which excludes all cases where $\rmA$ is not full rank and de facto excludes some of the applications considered in~\cite{ryu2019plug}) and it
imposes quite restrictive assumptions on relative values of $\mu$, $\varepsilon$ and $\Ltt$.

In a similar direction, Xu et al.~\cite{Xu2020}
very recently proposed a convergence study for PnP-ISTA, with the assumption that $\nabla F$ is $\Ly$-Lipschitz
with $\varepsilon \Ly \leq 1$. However, they assume that $D_\varepsilon$ is an
exact MMSE denoiser, \ie \
$D_{\vareps}(x_{\vareps}) = \mathbb{E}[x|x_{\vareps}]$, where $x \sim p$ and
$x_{\vareps} \sim \mathcal{N}(x, \vareps \Id)$ conditionally to $x$. Therefore
their theoretical results do not carry to many classical denoisers, such that
those learned from training data and implemented by neural networks.

\paragraph{Assumptions on algorithm parameters.}
Most of the convergence proofs for PnP algorithms impose restrictive
assumptions on the choice of parameters used in the iterative schemes.
This may exclude interesting ranges of parameters for several inverse
problems. For instance, for PnP-FBS, the parameter $\varepsilon$ (which can be
interpreted as the step of the proximal or gradient descents) and the Lipschitz
parameter $\Ly$ of $\nabla \Fdata$ must typically be chosen such that
$\Ly \varepsilon \leq C$ with $C\in [1,2]$ (see~\cite{ryu2019plug,Xu2020}, the
exact value of $C$ depends on the convergence proof). If
$\Fdata(\vx,\vy) = \frac{1}{2\alpha\sigma^2} \|\rmA \vx- \vy\|^2$, with
$\|\rmA\|\le 1$, it implies that
$\frac 1 \alpha \leq \frac{\sigma^2}{\varepsilon}$. The parameter $\varepsilon$
is imposed by the denoiser $D_{\vareps}$ (the denoiser is trained for a noise of
variance $\varepsilon$), and $\sigma$ is given by the quantity of noise in the
forward model.  If, for instance, the forward model involves a noise standard
deviation $\sigma$ which is 5 times smaller than the one used for the denoiser
$D_{\vareps}$, it means that the penalty $\alpha$ (which balances the respective
weights of the data and prior terms) should be chosen larger than $25$, which
implies that the algorithm will only converges for huge regularizations. We will
see in the experimental section that for this kind of reason the PnP-FBS
algorithm often fails to converge for classical imaging inverse problems, or
converges only for values of $\alpha$ which are not interesting in
practice. Fully proximal algorithms such as PnP-ADMM or PnP-BBS are much more
robust in practice, even when the conditions of their theoretical convergence
are not fully met. The PnP-SGD algorithm that we will introduce in the following
does not suffer from the same convergence limitations.

\paragraph{AMP algorithms.} It is worth mentioning at this point that the Plug-and-Play framework has also
been shown to be very efficient with Approximate Message Passing
algorithms~\cite{ahmad2020plug}. These algorithms have excellent convergence
properties for data terms of the form $\|\rmA\vx - \vy\|^2$ with $\rmA$ belonging to
specific classes of random matrices. This restriction on $\rmA$ does not hold for
the inverse problems considered in the current paper so we focus instead on
classical optimization scheme such as the ones described above.

\subsection{{{Plug \& Play} MAP estimators using gradient descent}} Now, assume that $D_\varepsilon = D^\star_\varepsilon$, where $D^\star_\varepsilon$ is the MMSE estimator to recover $x$ from the noisy observation $x_{\vareps} \sim \mathcal{N}(x, \vareps \Id)$ when $x$ has marginal density $p(x)$; that is,
\begin{equation}\label{eq:optimal_mmse_denoiser}
  \textstyle{
    D^\star_{\vareps}(x_{\vareps}) = \mathbb{E}[x|x_{\vareps}] = \int_{\rset^d} x p({x})  G_\vareps(x_{\vareps} - {x}) \rmd {x} / \int_{\rset^d} p({x})  G_\vareps(x_{\vareps} - {x}) \rmd {x}  \eqsp ,
    }
\end{equation} 
where $G_{\vareps}$ is a Gaussian kernel with variance $\varepsilon$. We introduce the following class of smooth approximations of $p(x)$, defined for any $x \in \rset^d$ by
\begin{equation}\label{eq:prior_epsilon}
\textstyle{
    p_\vareps(x) = \int_{\rset^d} p(\tilde{x}) G_\vareps(x - \tilde{x}) \rmd \tilde{x} \eqsp .}
\end{equation} 
In this case, Tweedie's
identity~\cite{efron2011tweedie} establishes the following relationship between the MMSE denoiser
$D^\star_\vareps$ and \eqref{eq:prior_epsilon}, for
any $x \in \rset^d$
\begin{equation}\label{eq:Tweedie}
  \textstyle{\nabla \Gprior_\varepsilon
  (\vx) = - \nabla \log p_\varepsilon(\vx) = (\vx - D^\star_\varepsilon(\vx))/
  \varepsilon \eqsp ,}
\end{equation}
where $U_{\vareps} = -\log(p_{\vareps})$.  This relation can be used to plug the MMSE denoiser $D^\star_\varepsilon$ in any gradient descent scheme involving $\nabla \Gprior_\varepsilon$ and it is at the
core of the algorithm PnP-SGD presented in this paper. Similarly to the MAP denoiser $D^\dagger_\varepsilon$, the MMSE denoiser $D^\star_\varepsilon$ is usually not known, so PnP methods rely on other denoisers $D_\varepsilon$ that are believed to be good approximations of $D^\star_\varepsilon$. Observe that approximating $D^\star_\varepsilon$ for realistic image priors is precisely the goal of CNN denoisers, while it is much more complicated to approach corresponding MAP denoisers $D^\dagger_\varepsilon$. This makes approaches based on Tweedie's identity particularly attractive. 

A similar relation is derived by Romano \emph{et al.} in~\cite{romano2017little}
where they present the Regularization by Denosing (RED) method, which proposes
an insightful Bayesian formulation of denoiser-based priors as image-adaptive
Laplacian regularisations. Instead of using Tweedie's identity, the RED method
solves equation~\eqref{eq:MAP} via different optimization algorithms (including
gradient descent and ADMM) with explicit regularization
$U_{\vareps}(\vx) = (1/2) \scalar{\vx}{\vx - D_\varepsilon(\vx)}.$ As shown in
\cite{reehorst2018regularization}, under the assumptions that $D_\varepsilon$ is
locally homogeneous and has symmetric Jacobian, this implies that for any
$x \in \rset^d$, $\nabla U_{\vareps} (\vx) = \vx - D_\varepsilon(\vx)$, which is
(up to a scaling factor $1/\varepsilon$) the same expression as Tweedie's
identity in~\eqref{eq:Tweedie}. Unfortunately, as pointed out before, these
assumptions on $D_\varepsilon$ are not strictly satisfied by most commonly used
denoisers \cite{reehorst2018regularization}, although we note that Jacobian
symmetry can be explicitely enforced \cite{Milanfar2013a}.  The convergence of
the RED algorithms for denoisers that do not verify the above-mentioned
assumption remains unproven.  As an alternative interpretation the RED algorithm
can be seen as a way to approximate the score $\nabla \Gprior_{\vareps}$ by
$ (\vx - {D}_\varepsilon(\vx))/ \varepsilon$ in the optimality equation
$\nabla \Fdata + \nabla \Gprior_{\varepsilon} = 0$. Here the optimal MMSE
denoiser ${D}^\star_\varepsilon$ is again replaced by some other denoiser.

More recently, \cite{cohen2020regularization} studies a projected RED estimator which seeks to minimise a data fidelity term subject to the constraint that the solution belongs to the set of fixed points $\{\vx \in \rset^d : \vx = D_\varepsilon(\vx)\}$, thus sharing strong link with the consensus equilibrium interpretation of proximal-based PnP estimators. It is reported in \cite{cohen2020regularization} that when $D_\varepsilon$ is a demi-contractive mapping, its fixed points define a convex set, which allows the construction of provably convergent algorithms for this alternative RED estimator. However, as pointed out in \cite{pesquet2020learning}, verifying that a given denoising operator is demi-contractive is not easy and, to be the best of our knowledge, it is not yet clear what denoisers verify this property. Furthermore, from a Bayesian inference viewpoint, additional studies would be required in order to determine when this projected RED estimator defines or approximates a MAP estimator for a suitable Bayesian model - we leave this as a perspective for future work.

The PnP-SGD optimisation algorithm that will be presented in the next section is very close
to the gradient descent version of RED  presented in~\cite{romano2017little}. We will show that it converges to the vicinity of the solution of~\eqref{eq:MAP}
under much milder conditions than previously assumed, and in particular when
$D_\varepsilon$ is not an exact MAP or MMSE denoiser. Importantly, our
convergence proof is valid for the neural network denoiser used in
\cite{ryu2019plug} (a variant of DnCNN~\cite{zhang2017beyond} with a contractive
residual) and also for the native Non Local Means~\cite{buades2005review}. %

\section{PnP maximum-a-posteriori estimation: analysis and computation}
\label{sec:pnp-theory-alg}

\subsection{Analysis of maximum-a-posteriori estimation with PnP priors}
\label{sec:pnp-sgd}

\label{sec:sgd}
We are interested in MAP estimation for Bayesian models involving PnP priors
that are defined implicitly by an image denoising algorithm $D_\vareps$. We pay
special attention to the highly practically relevant case in which $D_\vareps$
approximates the optimal MMSE denoiser $D_\vareps^\star$ associated to $p$,
i.e., $D_\vareps^\star = \mathbb{E}[\vx|\x_{\vareps}]$ for
$x_{\vareps} \sim \mathcal{N}(x, \vareps \Id)$ when $\vx$ has marginal density
$p$. As mentioned previously, state-of-the-art denoisers based on neural
networks are often trained to approximate $D_\vareps^\star$ by using a sample of
clean images $\{x_i\}_{i=1}^N$ from $p$, a corresponding noisy sample
$\{x_i'\}_{i=1}^N$ with $x_i' \sim \mathcal{N}(x_i, \vareps \Id)$, and choosing
$D_\vareps$ to approximately minimize the empirical MSE loss
$\sum_{i=1}^N\normLigne{D_\vareps(x_i') - x_i}^2$. Similarly, many
state-of-the-art patch-based image denoisers are also designed to approximate
$D_\vareps^\star$.

The fact that $D_\vareps$ is only an approximation of $D_\vareps^\star$ leads to
several complications in the analysis and computation of MAP solutions. For
example, unlike $D_\vareps^\star$, $D_\vareps$ does not define a gradient
mapping in general, and key results such as Tweedie's
identity~\cite{efron2011tweedie} do not hold. Moreover, in the case of neural
network denoisers trained from samples $\{x_i\}_{i=1}^N$ from $p$, the model is
unknown as it is only available through  $\{x_i\}_{i=1}^N$,
making it difficult to check that basic regularity properties required for MAP
estimation are satisfied.

Rather than imposing strong assumptions on $D_\vareps$, we address these difficulties by formulating our analysis in the \textit{M-complete} Bayesian framework, in which we assume that the posterior $p(x|y)$ associated with the true prior $p(x)$ exists but remains largely unknown, and all inference on $\vx|\vy$ are conducted by using operational approximations of this true model \cite{bernardo_smith_bayesian_theory}. In particular, we focus on the class of smooth approximations of $p(x|y)$ given for any $\vareps > 0$ and $x \in \rset^d$ by
\begin{equation}
  \label{eq:posterior_eps}
  p_\vareps(x|y) = \frac{p_\vareps(x)p(y|x)}{\int_{\rset^d}p_\vareps(\tilde{x})p(y|\tilde{x}) \rmd \tilde{x}} \eqsp , 
\end{equation}
where $p_\vareps(x)$ is the smooth approximation of the prior $p(x)$ defined in
\eqref{eq:prior_epsilon}. We will study MAP estimation for $p_\vareps(x|y)$ to
establish that the procedure is well defined, well posed, amenable to efficient
computation, and that it provides a useful approximation to MAP estimation with
the true posterior $p(x|y)$. Following on from this,
Section~\ref{sec:theoretical-analysis} will study the computation of MAP
solutions for $p_\vareps(x|y)$ by using PnP SGD with a generic denoiser
$D_\vareps$ that approximates $D_\vareps^\star$, where we will pay particular
attention to the conditions on $D_\vareps$ required to ensure convergence, as
well as to the bias introduced by using $D_\vareps$ instead of
$D_\vareps^\star$.

It is established in \cite{laumont2020pnpula} that, under basic assumptions on the likelihood function $p(y|x)$ detailed in \textup{H\ref{assum:post}} below, the posterior approximation $p_\vareps(x|y)$ is well defined, proper, and can be made as close to the true posterior $p(x|y)$ as desired by reducing the value of $\vareps$, with the approximation error vanishing as $\vareps \rightarrow 0$. Crucially, \cite{laumont2020pnpula} also establishes that, under \textup{H\ref{assum:post}} and mild assumptions on the optimal MMSE denoiser $D^\star_\vareps$ (essentially, that the denoising problem underlying $D^\star_\vareps$ is well posed in the sense of Hadamard), then $x \mapsto p_\vareps(x|y)$ is differentiable with $x \mapsto \nabla \log p_\vareps(x|y)$ Lipschitz continuous. We conclude that the approximation $p_\vareps(x|y)$ is well defined and amenable to computation by first-order schemes, such as SGD to compute critical points of $p_\vareps(x|y)$ and perform MAP estimation.

\begin{assumption}
\label{assum:post}
For any $y \in \rset^\dimY$, $\sup_{x \in \rset^d} p(y|x) < +\infty$,
$p(y|\cdot) \in \rmc^1(\rset^d, \ooint{0, +\infty})$. In addition, there exists
$\Ltt_y > 0$ such that $\nabla \log p(y|\cdot)$ is $\Ltt_y$ Lipschitz continuous
and $x \mapsto \log p(y|x)$ is real-analytic
\footnote{A function $f: \ \rset^d \to \rset$ is said to be real-analytic if for any
$x_0 = (x_0^1, \dots, x_0^d) \in \rset^d$ there exists
$(a_{n_1, \dots, n_d})_{n_1, \dots, n_d \in \nset} \in \rset^{\nset^d}$ and
$r > 0$ such that for any $x = (x^1, \dots, x^d) \in \ball{x_0}{r}$
\begin{equation*}
  \textstyle{
    f(x) = \sum_{n_1 \in \nset} \dots \sum_{n_d \in \nset} a_{n_1, \dots, n_d} \prod_{j=1}^d (x^j - x_0^j)^{n_j} \eqsp .
    }
  \end{equation*}}
\footnote{The assumption that
$x \mapsto \log(p(y|x))$ is real-analytic is satisfied in all of our experiments
since there exists $\rmA \in \rset^{p \times d}$ and $\sigma > 0$
such that for any $x \in \rset^d$ and $y \in \rset^p$,
$\log p(y|x) = \normLigne{\rmA x - y}^2 / (2 \sigma^2)$.}
\footnote{From Liouville's theorem one could think that the simultaneously verifying that $\nabla \log p(y|\cdot)$ is Lipschitz continuous and that $x \mapsto \log(p(y|x))$ is real-analytic restricts our analysis to models for which $\nabla^2 \log p(y|\cdot)$ is constant (i.e., Gaussian models), but this is not the case because Liouville's theorem applies entire functions, which are a subclass of the real-analytic class.}.
\end{assumption}

With the above-mentioned properties of $p_\vareps(x|y)$ in mind, we wonder if computing a MAP solution for $p_\vareps(x|y)$ provides useful information about a MAP solution for $p(x|y)$. More precisely, we study if critical points for $p_\vareps(x|y)$ are stable w.r.t. variations in $\vareps$, and if they converge to critical points of $p(x|y)$ as $\vareps \rightarrow 0$. Proposition \ref{prop:pnp_sgd_epsilon} below establishes that this is indeed the case. In words, MAP solutions computed with $p_\vareps(x|y)$ are in the neighbourhood of MAP solutions for $p(x|y)$, with $\vareps$ controlling a trade-off between the computational efficiency of first-order schemes and the accuracy of the delivered solutions w.r.t. $p(x|y)$.

Formally, we investigate the dependency of the set of stationary points
$\mss_{\vareps, \msk} = \ensembleLigne{x \in \rset^d}{\nabla \log
  p_{\vareps}(x|y) = 0}$ w.r.t. $\vareps > 0$.  We show
that each cluster point (in the sense of the Hausdorff distance, see below) of
the sequences of sets $(\mss_{\vareps_n, \msk})_{n \in \nset}$ with
$\lim_{n \to +\infty} \vareps_n = 0$ is contained in the set of stationary
points of $x \mapsto p(x|y)$, \ie \ the true posterior, denoted by $\mss_{\msk} = \ensembleLigne{x \in \msk}{\nabla \log p(x|y) = 0}$.

We start by recalling that for any compact set $\msc \subset \rset^{\dim}$, we
have that $\mathcal{K}_{\msc} = \ensembleLigne{\msk}{\text{$\msk$ is compact and
    $\msk \subset \msc$}}$. $(\mathcal{K}_{\msc}, \rmd_{\msc})$ is a metric
space where the metric $d_{\msc}$ is called the Hausdorff distance $\rmd_{\msc}$ and 
is given for any $\msk_1, \msk_2 \in \mathcal{K}_{\msc}$ by
\begin{equation}
  \label{eq:hausdorff}
  \rmd_{\msc}(\msk_1, \msk_2) = \inf \ensembleLigne{\vareps > 0 }{\msk_1 \subset \msk_2 + \cball{0}{\vareps} \eqsp, \msk_2 \subset \msk_1 + \cball{0}{\vareps}} \eqsp ,
\end{equation}
where for any pair of sets $\msa, \msb \subset \rset^{\dim}$ we have
$\msa + \msb = \ensembleLigne{x + y}{x \in \msa, y \in \msb}$. 

Denote $\mse$ the set of sequences $(\vareps_n)_{n \in \nset}$ such that for any
$n \in \nset$, $\vareps_n >0$ and $\lim_{n \to +\infty} \vareps_n = 0$.  For any
sequence $e = (\vareps_n)_{n \in \nset} \in \mse$ denote $\mst_{e, \msk}$ the
set of cluster points of $(\mss_{\vareps_n, \msk})_{n \in \nset}$ with respect
to the Haussdorff distance on $\msk$, $\rmd_{\msk}$ defined in
\eqref{eq:hausdorff}.  We also define
$\mst_{\msk} = \bigcup_{e \in \mse} \mst_{e, \msk}$, \ie \ the collection of the
cluster points for all the sequences $(\vareps_n)_{n \in \nset}$ such that
$\lim_{n \to +\infty} \vareps_n = 0$.  Finally, we define
$\mss_{\msk}^\star = \bigcup_{\mss \in \mst_{\msk}} \mss$ the union of all the
cluster points.

\begin{proposition}
  \label{prop:pnp_sgd_epsilon}
  Assume \textup{H\ref{assum:post}} and that
  $p \in \rmc^1(\rset^d, \ooint{0,+\infty})$ with
  $\normLigne{p}_\infty + \normLigne{\nabla p}_\infty< +\infty$.  Then for any
  compact set $\msk$, $\mss_{\msk}^\star \subset \mss_{\msk}$ with
  $\mss_{\msk} = \ensembleLigne{x \in \msk}{\nabla \log p(x|y) = 0}$.
  \end{proposition}

\begin{proof}
  Let $(\vareps_n)_{n \in \nset} \in \mse$ and $\mss$ a cluster point of
  $(\mss_{\msk, \vareps_n})_{n \in \nset}$. Without loss of generality we assume
  that $\lim_{n \to +\infty} \mss_{\msk, \vareps_n} = \mss$.  Let
  $x^{\star} \in \mss$.  For any $\eta > 0$ there exists $n_{\eta} \in \nset$
  such that for any $n \in \nset$ with $n \geq n_\eta$,
  $\mss \subset \mss_{\msk, \vareps_n} + \cball{0}{\eta}$. Hence, for any
  $n \in \nsets$ there exist an increasing sequence
  $(k_n)_{n \in \nset} \in \nset^{\nset}$ and
  $x_n \in \mss_{\msk, \vareps_{k_n}}$ and $z_n \in \cball{0}{1/n}$ such that
  $x^{\star} = x_{k_n} + z_n$. Since $\lim_{n \to +\infty} z_n = 0$ we get that
  $\lim_{n \to +\infty} x_{k_n} = x^\star$.

  In what follows, we show that $\lim_{n \to +\infty} \nabla \log(p_{\vareps_{k_n}}(x_{k_n})) = \nabla \log p(x^{\star})$.
  First, we show that
    \begin{equation}
      \lim_{n \to +\infty} \max(\absLigne{p - p_{\vareps_{k_n}}}_{\infty, \msk}, \normLigne{\nabla p - \nabla p_{\vareps_{k_n}}}_{\infty, \msk}) = 0 \eqsp .
    \end{equation}
    Indeed, let $f \in \rmc(\rset^d, \rset^p)$ with $p \in \nset$ such that
    $\normLigne{f}_\infty < +\infty$ and denote
    $f_\vareps \in \rmc(\rset^d, \rset^p)$ given for any $x \in \rset^d$ by
    \begin{equation}
      \textstyle{f_\vareps(x) = \int_{\rset^d} f(y) G_\vareps(x-\tilde{x}) \rmd \tilde{x} \eqsp , }
    \end{equation}
    where we recall that for any $\tilde{x} \in \rset^d$,
    $G_\vareps(\tilde{x}) = (2 \uppi \vareps)^{-d/2}
    \exp[-\normLigne{\tilde{x}}^2/(2 \vareps)]$. For ease of notation, we define
    $G = G_1$. Let $\eta > 0$. Then, there exists $R > 0$ such that for any
    $\vareps > 0$ we have
    \begin{equation}
      \label{eq:convo_1}
      \textstyle{ \int_{\normLigne{\tilde{x}} > R} \normLigne{f(x - \vareps^{1/2} \tilde{x}) - f(x)} G(\tilde{x}) \rmd \tilde{x} \leq 2 \normLigne{f}_\infty \int_{\normLigne{\tilde{x}} > R}  G(\tilde{x}) \rmd \tilde{x} < \eta / 2 \eqsp . }
    \end{equation}
    Let $\msk' = \msk + \cball{0}{R}$. We have that $\msk'$ is compact and
    therefore $f$ is uniformly continuous on $\msk'$. Hence there exists
    $\xi > 0$ such that for any $x \in \msk$, $\vareps \in \ocintLigne{0, \xi}$ and
    $y \in \cball{0}{R}$ we have
    \begin{equation}
      \label{eq:convo_2}
      \absLigne{f(x - \vareps^{1/2}y) - f(x)} \leq \eta/2 \eqsp . 
    \end{equation}
    Hence, combining \eqref{eq:convo_1} and \eqref{eq:convo_2} we get that for
    any $x \in \msk$ and $\vareps \in \ocintLigne{0, \xi}$
    \begin{align}
      &\normLigne{f_\vareps(x) - f(x)} \leq \textstyle{ \int_{\rset^d} \normLigne{f(x - \tilde{x}) - f(x)} G_\vareps(\tilde{x}) \rmd \tilde{x} } \\
      &\qquad \quad  \leq \textstyle{ \int_{\rset^d} \normLigne{f(x - \vareps^{1/2} \tilde{x}) - f(x)} G(\tilde{x}) \rmd \tilde{x} } \\
      &\qquad \quad  \leq \textstyle{ \int_{\cball{0}{R}} \normLigne{f(x - \vareps^{1/2} \tilde{x}) - f(x)} G(\tilde{x}) \rmd \tilde{x}} \\
      &\qquad \quad  \qquad + \textstyle{\int_{\cball{0}{R}^\complementary} \normLigne{f(x - \vareps^{1/2} \tilde{x}) - f(x)} G(\tilde{x}) \rmd \tilde{x} }  \\
      &\qquad \quad  \leq \eta/2 + \textstyle{\int_{\cball{0}{R}} \normLigne{f(x - \vareps^{1/2} \tilde{x}) - f(x)} G(\tilde{x}) \rmd \tilde{x}} \leq \eta \eqsp . 
    \end{align}
    Hence $\lim_{\vareps \to 0} \normLigne{f - f_\vareps}_{\infty, \msk} =
    0$. Therefore using this result and that $p \in \rmc^1(\rset^d, \rset)$ with
    $\normLigne{p}_\infty + \normLigne{\nabla p}_\infty < +\infty$ we get that
        \begin{equation}
      \lim_{n \to +\infty} \max(\absLigne{p - p_{\vareps_{k_n}}}_{\infty, \msk}, \normLigne{\nabla p - \nabla p_{\vareps_{k_n}}}_{\infty, \msk}) = 0 \eqsp .
    \end{equation}
    Combining this result, the fact that
    $\lim_{n \to +\infty} x_{k_n} = x^\star$ and that $p^\star > 0$, we get that
    $\lim_{n \to +\infty} \nabla \log(p_{\vareps_{k_n}}(x_{k_n})) = \nabla \log
    p(x^{\star})$. Finally, we obtain that
    \begin{equation}
      \lim_{n \to +\infty} \defEns{\nabla \log p(y|x_{k_n})  + \nabla \log p_{\vareps_{k_n}}(x_{k_n}) } = \nabla \log p(y|x^\star) + \nabla \log p(x^\star) = 0 \eqsp .
    \end{equation}
    Hence, $x^\star \in \mss_{\msk}$ and therefore $\mss_{\msk}^\star \subset \mss_{\msk}$. \qed
\end{proof}

As a third and final point in our analysis, we study if MAP estimation for
$p_\vareps(x|y)$ is a well-posed estimation procedure, which is an essential
requirement for meaningful inference. One would ideally seek to establish the
existence of a unique global maximiser that is Lipschitz continuous
w.r.t. perturbations of the observed data $y$. Unfortunately, this is not
possible without imposing very strong assumptions on the model. Instead,
Proposition~\ref{prop:stab_map} below shows that, under some assumptions on the
likelihood $p(y|x)$, the set of critical points of $p_\vareps(x|y)$ is locally
Lipschitz continuous w.r.t. perturbations of $y$, which is a weaker form of
well-posedness. Notice that the assumptions on the likelihood can be relaxed
when $D_\vareps^\star$ is contractive, but this is usually unrealistic. This
highlights a limitation of MAP estimation by comparison to other Bayesian
estimators, namely MMSE estimation, which is shown in \cite{laumont2020pnpula}
to be well-posed under significantly weaker assumptions.

\begin{proposition}
  \label{prop:stab_map}
  Assume \textup{H\ref{assum:post}}, that
  $(x,y) \mapsto p(y|x) \in \rmc^2(\rset^d \times \rset^p, \rset)$ and
  $p \in \rmc(\rset^d, \rset_+)$. Let $y_0 \in \rset^p$ denote some observed
  data and $x^\star_{y_0}\in \rset^d$ a maximiser of the posterior
  $x \mapsto p(x|y_0)$.  In addition assume either that the Hessian matrix
  $\nabla^2_{x} \log p(y_0|x_0^\star)$ is positive definite, or that the
  Jacobian $\normLigne{\rmd D_\vareps^\star(x_{y_0}^\star)} <1$. Then there
  exists an open set $\msv_0 \subset \rset^p$ and a function
  $x^\star(y) \in \rmc^1(\msv_0, \rset^d)$ such that $y_0 \in \msv_0$ and for
  any $y \in \msv_0$, $x^\star(y)$ is a local minimizer of $x \mapsto p(x|y)$.
\end{proposition}
\begin{proof}
  First, using that $p \in \rmc(\rset^d, \rset_+)$ we have that for any
  $v \in \rset^d$ and $c \in \rset$ there exists $\msa \in \mcb{\rset^d}$ such
  that $\int_\msa \abs{\langle x, v \rangle -c}p(x) \rmd x > 0$. Hence, we can
  apply \cite[Lemma II.1]{gribonval2011should} and
  $D_\vareps \in \rmc^\infty(\rset^d, \rset^d)$.
  
  Note that
  $(x,y) \mapsto p(x|y) \in \rmc^2(\rset^d \times \rset^p, \rset_+)$.  Since
  $\nabla^2 \log p(x_{y_0}^\star|x_0)$ is positive there exists
  $\msu_1 \subset \rset^d$ open and $ \msv_1 \subset \rset^p$ open such that for
  any $x \in \msu_1$ and $y \in \msv_1$, $\nabla^2 \log p(y|x)$ is
  positive. Hence, for any $y \in \msv_1$, $x \in \rset^d$ is a local minimizer
  if and only $\nabla \log p(y|x) = 0$.

  Let $F \in \rmc^1(\msu_1 \times \msv_1, \rset^d)$ given for any
  $x \in \rset^d$ and $y \in \rset^p$ by
  \begin{equation}
    F(x,y) = \nabla_x \log p(x|y) = (x - D_\vareps^\star(x))/\vareps + \nabla \log p(y|x) \eqsp . 
  \end{equation}
  Using that either $\normLigne{\rmd D_\vareps^\star(x_{y_0}^\star)} \leq 1$ or that
  $\nabla^2 \log p(y_0|x_0^\star)$ is positive, we get that
  $\nabla_x F(x_{y_0}^\star, y_0)$ is invertible. Therefore using the implicit
  function theorem, there exists $\varphi \in \rmc^1(\msv_0, \rset^d)$ such that
  for any $y \in \msv_0$, $F(\varphi(y), y) =0$, \ie \ $\varphi(y)$ is a local
  minimizer of $x \mapsto \log p(y|x)$ which concludes the proof.\qed
\end{proof}

To conclude, a major challenge in understanding Bayesian inference with PnP priors and providing guarantees for the delivered solutions is that the underlying prior and posterior densities $p(x)$ and $p(x|y)$ are unknown. Also, the image denoiser $D_\vareps$ used to construct PnP schemes is not usually directly related to the model. Instead, when it approximates the optimal MMSE denoiser $D^\star_\vareps$, it is indirectly related to the model via Tweedie's identity and the smooth approximations $p_\vareps(x)$ and $p_\vareps(x|y)$. We establish that these operational approximations are useful for MAP inference for $x|y$, in the sense that they are well defined, proper, and MAP solutions for $p_\vareps(x|y)$ can be made arbitrarily close to the true MAP solutions through the choice of $\vareps$. Importantly, under some assumptions, MAP solutions for $p_\vareps(x|y)$ are well posed and amenable to efficient computation by first order optimisation methodology.

\subsection{PnP-SGD and convergence}
\label{sec:theoretical-analysis}

We are now ready to study the computation of MAP solutions for $p_\vareps(x|y)$ by using PnP SGD with a generic denoiser $D_\vareps$ that approximates $D_\vareps^\star$. We pay particular particular attention to the conditions on $D_\vareps$ required to ensure convergence, and to the bias introduced by using $D_\vareps$ instead of $D_\vareps^\star$.

We begin by using Tweedie's identity to express SGD to compute critical points of $p_\vareps(x|y)$ as the
following sequence: $X_0 \in \rset^d$ and for any $k \in \nset$
\begin{equation}\label{eq:SGD_parfait}
X_{k+1} = X_{k} - \delta_k \nabla \Fdata (X_k,\vy) - \delta_k/\vareps (X_k - D_\vareps^\star(X_k))+ {\delta_k} Z_{k+1} \eqsp ,
\end{equation}
where $(\delta_k)_{k \in \nset} \in (\rset_+)^{\nset}$ is a family of
step-sizes, $\vareps >0$, 
and $\ensembleLigne{Z_k}{k \in \nset}$ a family of i.i.d. Gaussian
random variables with zero mean and identity covariance matrix.  We
recall that the sequences $(X_k)_{k \in \nset}$ and
$(Z_k)_{k \in \nset}$ are defined on an underlying probability space
$(\Omega, \mcf, \mathbb{P})$.

As mentioned previously, in most practically relevant cases $D_\vareps^\star$ is an abstract quantity that cannot be computed. Instead, we have a different denoiser $D_\vareps$ that can be assumed to be a good approximation of $D_\vareps^\star$.
For example, when we have access to samples $\{x_i\}_{i=1}^N$
from $p$ we can consider a noisy version of these samples $\{x_i'\}_{i=1}^N$
with level $\vareps > 0$ and train a neural network based denoiser $D_\vareps$
to minimize the loss $\sum_{i=1}^N\normLigne{D_\vareps(x_i') - x_i}^2$. This
loss corresponds to the empirical version of
$\expeLigne{\normLigne{D_\vareps(x_\vareps) - x}^2}$ (with $x \sim p$ and
$x_{\vareps} \sim \mathcal{N}(x, \vareps \Id)$ conditionally to $x$) whose
minimizer is the MMSE $D_\vareps^\star$. 

Using a generic denoiser $D_\vareps$ in our SGD scheme in lieu of
$D_\vareps^\star$ we obtain the Plug \& Play SGD algorithm associated with
following recursion: $X_0 \in \rset^{\dim}$ and for any $k \in \nset$
\begin{align}
  \label{eq:pnpsgd}
  X_{k+1} &= X_k + \delta_k (b_{\vareps}(X_k)  +  Z_{k+1}) \eqsp , \\
  b_{\vareps}(x) &= \nabla \log(p(y|x)) + \alpha (D_{\vareps}(x) - x) / \vareps \eqsp ,
\end{align}
where we note that we have introduced a regularization parameter $\alpha > 0$ that controls the amount of regularisation enforced by $D_{\vareps}$. The original SGD algorithm is recovered by setting $\alpha=1$ and $D_{\vareps}=D^\star_{\vareps}$.

\begin{algorithm}
	\caption{PnP-SGD}\label{algo:pnp-sgd}
	\begin{algorithmic}

		\REQUIRE $n, n_{\mathtt{burnin}} \in \nset$, $y \in \rset^\dimY$, $\vareps,  \alpha, \delta > 0$
		
		\STATE \textbf{Initialization:} Set $X_0 = \tilde{x}$ and $k=0$.
		
		\FOR{$k=0:N$}
		
		\STATE $Z_{k+1} \sim \mathcal{N}(0, \Id)$
		
		\IF{$k\leq n_{\mathtt{burnin}}$}
		
		\STATE $X_{k+1} = X_k + \delta_0 \nabla \log(p(y|X_k)) + (\delta_0 \alpha / \vareps) (D_{\vareps}(X_k) -X_k)  + \delta_0 Z_{k+1}$

		\ENDIF

		\IF{$k\geq n_{\mathtt{burnin}}$}

		\STATE $X_{k+1} = X_k + \delta_k \nabla \log(p(y|X_k)) + (\delta_k \alpha / \vareps) (D_{\vareps}(X_k) -X_k)  + \delta_k Z_{k+1}$

		\STATE $\delta_{k+1} = \delta_0 (k+1-n_{\mathtt{burnin}})^{-0.8}$
		
		\ENDIF

		\ENDFOR
		
		\STATE \textbf{return} $X_{N}$

	\end{algorithmic}
\end{algorithm}

We now turn to the proof of convergence of \pnpsgd .
The asymptotic estimates we derive in this work are only valid for sequences
which remain in a compact set $\msk$, which is a classical assumption in
stochastic approximation
\cite{tadic:doucet:2017,delyon1999convergence,delyon1996general,metivier1984applications}.
Under tighter conditions on $x \mapsto \log p_\vareps(x|y)$ this limitation can
be circumvented using the global asymptotic results of \cite[Theorem
A1.1]{tadic:doucet:2017}. Another way to remove this restriction would be to
consider an additive term of the form $x \mapsto (x - \Pi_{\msc}(x)) / \lambda$
in $b_\vareps$ (where $\Pi_{\msc}$ is the projection onto some compact convex
set $\msc$ and $\lambda > 0$ some hyperparameter) which ensures the stability of
the numerical scheme. We leave this analysis for future work. In practice, we have not observed any stability issues for PnP-SGD provided that the stepsize is chosen appropriately see Section~\ref{sec:convergence_conditions}.

In what follows, we show that the bias of \pnpsgd \ depends on
the distance between $D_{\vareps}$ and the MMSE estimator $D^\star_{\vareps}$, using recent results from
\cite{tadic:doucet:2017}. %

\begin{assumption}
  \label{assum:neural_net}
  Assume that there exist $\bvareps > 0$, $\Ltt \geq 0$ and a function
  $\Mtt: \mathbb{R}^+ \to \mathbb{R}^+$ such that for any
  $\vareps \in \ocint{0, \bvareps}$, $R \geq 0$, $x_1, x_2 \in \rset^d$ and
  $x \in \cball{0}{R}$ we have
  \begin{equation}
    \label{eq:cond_neural_net}
      \norm{D_{\vareps}(x_{1}) - D_{\vareps}(x_{2})} \leq \Ltt \norm{x_{1} - x_{2}} \eqsp , \qquad  
      \norm{D_{\vareps}(x) - D_{\vareps}^{\star}(x)} \leq \Mtt(R) \eqsp ,
  \end{equation}
  where we recall that \begin{equation}
    \label{eq:def_d_star}
    \textstyle{
      D_{\vareps}^{\star}(x_{1}) =  \int_{\rset^{\dim}} \tilde{x} \ g_{\vareps}(\tilde{x} | x_{1}) \rmd \tilde{x} \eqsp ,
      }
    \end{equation}
    with $\tilde{x} \mapsto g_\vareps(\tilde{x}|x)$ the probability density of
    $X$ given $X_\vareps=x$ where $X_\vareps \sim \mathcal{N}(X, \vareps \Id)$
    conditionally to $X$ and $X \sim p$.
  \end{assumption}

  The first part of \eqref{eq:cond_neural_net} regarding the smoothness property
  of the denoiser can be explicitly verified for a certain class of neural
  networks by adding a spectral regularization term for each layer of the neural
  network, see \cite{ryu2019plug,miyato2018spectral}. The second condition follows from carefully selecting the loss of the neural network
  as in the Noise2Noise network introduced in \cite{lehtinen2018noise2noise} and
  controlling the population error, see \cite{laumont2020pnpula}.

  We are now ready to state Proposition \ref{prop:cv_pnpsgd} which ensures that
  stable PnP-SGD sequences are close to the set of stationary points of
  $x \mapsto \log p_\vareps(x|y)$ where $x \mapsto \log p_\vareps(x|y)$ is given
  in \eqref{eq:posterior_eps}. The distance to this set of stationary points is
  controlled by the approximation error of the network $D_\vareps$.

\begin{proposition}
  \label{prop:cv_pnpsgd}
  Assume \textup{H\ref{assum:post}}, \textup{H\ref{assum:neural_net}}. Let
  $\alpha > 0$ and $\vareps \in \ocint{0, \vareps_0}$.  Assume that
  $\lim_{k \to +\infty} \delta_k = 0$, $\sum_{k \in \nset} \delta_k = +\infty$
  and $\sum_{k \in \nset} \delta_k^2 < +\infty$. Let $R > 0$,
  $\msk \subset \cball{0}{R}$ be a compact set, $X_0 \in \rset^{\dim}$ and
  $\msa_{\vareps, \msk} \in \mcf$ given by
  \begin{equation}\msa_{\vareps, \msk} = \ensembleLigne{\omega \in
    \Omega}{\text{there exists $k_0 \in \nset$ such that for any
      $k \geq k_0$, $X_k(\omega) \in \msk$.}} \eqsp ,
\end{equation}
where $(X_k)_{k \in \nset}$ is given by \eqref{eq:pnpsgd}.  Then there exist
$C_{\vareps, \msk} \geq 0$ and $r_{\vareps, \msk} \in \ooint{0,1}$ such that
$\limsup_{k \to +\infty} d(X_k(\omega), \mss_{\vareps, \msk}) \leq C_{\vareps,
  \msk} \Mtt(R)^{r_{\vareps, \msk}}$ for any $\omega \in \msa_{\vareps, \msk}$,
with
\begin{equation}
 \mss_{\vareps, \msk} = \ensemble{x \in \msk}{\nabla \log p_\vareps(x|y)  = 0} \eqsp ,
\end{equation}
where $x \mapsto p_\vareps(x|y)$ is given in \eqref{eq:posterior_eps}.
\end{proposition}

\begin{proof}
  Let $\vareps >0$ and $\omega \in \msa_{\vareps, \msk}$. For any $k \in \nset$,
  let $\zeta_k = Z_{k+1}$ and
  $\eta_k = b_{\vareps}(X_k) - \nabla \log p(y|X_k) - \nabla \log
  p_{\vareps}(X_k)$.  Using H\ref{assum:neural_net} we have for any
  $k \in \nset$,
  \begin{equation}
   \normLigne{b_\vareps(X_k) - \nabla \log p(y|X_k) - \nabla \log
    p_{\vareps}(X_k)} = \vareps^{-1} \normLigne{D_{\vareps}(X_k) -
    D_{\vareps}^{\star}(X_k)} \leq \Mtt(R)/\vareps \eqsp . 
  \end{equation}
  Hence, we obtain that \cite[Assumption 2.1, Assumption 2.2]{tadic:doucet:2017}
  are satisfied.  In what follows, we show that \cite[Assumption
  2.3.c]{tadic:doucet:2017} holds.  First, we introduce
  $G_\vareps: \ \rset^d \to \rset$ given for any $x \in \rset^d$ by
  \begin{equation}
   G_{\vareps}(x) = (2\uppi \vareps)^{-d/2} \exp[- \normLigne{x}^2/(2\vareps)]  \eqsp . 
  \end{equation}
  We have that for any $x \in \rset^{\dim}$,
  $p_{\vareps}(x) = (p \ast G_{\vareps})(x)$, where $\ast$ denotes the
  convolution product. Since $p, G_{\vareps} \in \mathrm{L}^1(\rset^{\dim})$ we
  get that for any $\xi \in \rset^{\dim}$,
  $\widehat{p \ast G_{\vareps}}(\xi) = \hat{p}(\xi)
  \hat{G_{\vareps}}(\xi)$. Since $p \in \mathrm{L}^1(\rset^{\dim})$,
  $\| \hat{p} \|_{\infty} < +\infty$ using Riemann-Lebesgue theorem and in
  addition $\hat{G}_{\vareps}(\xi) = \exp[- \vareps \norm{\xi}^2 / 2]$. Hence,
  $\widehat{p \ast G_{\vareps}} \in \mathrm{L}^1(\rset^{\dim})$ and we obtain
  that  for almost every $x \in \rset^{\dim}$
  \begin{equation}
    p_{\vareps}(x) = \textstyle{\int_{\rset^{\dim}} \hat{p}(\xi)
  \hat{G_{\vareps}}(\xi) \exp[\rmi \langle x, \xi \rangle] \rmd \xi \eqsp .}
\end{equation}
In the rest of the proof, we denote
$\bar{p}_{\vareps} : \ \cset^{\dim} \to \cset$ given for any
$z=(z^1, \dots, z^d) \in \cset^{\dim}$ by
$\bar{p}_{\vareps}(z) = \int_{\rset^{\dim}} \hat{p}(\xi) \hat{G_{\vareps}}(\xi)
\exp[\rmi \langle z, \xi \rangle] \rmd \xi$ where for any $z_1, z_2 \in \cset^d$
we have $\langle z_1, z_2 \rangle = \sum_{j=1}^d z_1^j \bar{z}_2^j$.  We have
that $\bar{p}_{\vareps}$ is analytic using the dominated convergence
theorem. Since for any $x \in \rset^{\dim}$, $p_{\vareps}(x) > 0$ and
$\bar{p}_{\vareps} \in \rmc(\cset^{\dim}, \cset)$, there exists an open set
$\msu \subset \cset^{\dim}$ such that for any $z \in \msu$,
$\Re(\bar{p}_{\vareps}(z)) > 0$. Since
$\log : \cset \backslash (\ensembleLigne{t \in \cset}{\Re(t) \leq 0}) \to \cset$
is analytic we obtain that $z \mapsto \log \bar{p}_{\vareps}(z)$ is analytic on
$\msu$. Hence, $x \mapsto \log p(y|x) + \log p_{\vareps}(x)$ is real-analytic on
$\rset^{\dim}$.  We conclude using \cite[Theorem 2.1]{tadic:doucet:2017}.
\end{proof}

The proof can be extended to the case where $Z_k =0$ using \cite[Theorem
2.1]{tadic:doucet:2017}. In this case the assumption that
$\sum_{k \in \nset} \delta_k^2 < +\infty$ can be replaced by
$\lim_{k \to +\infty} \delta_k = 0$.

The following experimental section demonstrates the PnP-SGD algorithm
on three canonical imaging inverse problems, namely image deblurring,
inpainting, and denoising, along with other standard PnP algorithms. 

\section{Experimental study}\label{sec:experimental}
\label{sec:experiments}

In this section, we study the behaviour of several PnP algorithms for three
classical inverse problems: denoising, deblurring and inpainting.  We recall
that in each of these problems we consider a prior model
$p(x) \propto \exp[-U(x)]$ which is unknown and that the inference $x|y$ is
obtained by approximation of this model. For the deblurring and denoising
problems, the log-posterior of the degradation model can be written for any
$x, y \in \rset^d$ as
\begin{equation}
\label{logposterior_exp}
  - \log p(\vx|\vy) = \|\rmA \vx - \vy\|^2/(2\sigma^2) + \alpha \Gprior(x) + C \eqsp ,
\end{equation}
where $\rmA$ is a $d \times d$ matrix, $C \geq 0$ is a constant and the
parameter $\alpha \geq 0 $ balances the weights of the log-likelihood
$\Fdata(\vx,\vy)$ and the log-prior $U$. In this case, we have for any
$x, y \in \rset^d$, $F(x,y) = \|\rmA \vx - \vy\|^2/(2\sigma^2)$.  In our
inpainting experiments, we change the likelihood so that pixels are either
visible or hidden. In this case the log-posterior can be written for any
$x \in \rset^d$ and $y \in \rset^m$ as
\begin{equation}
\label{logposterior_inpainting}
  - \log p(\vx|\vy) =
  \iota_{\rmQ \vx =  \vy} + \alpha \Gprior(\vx) + C \eqsp ,   
\text{ with } \iota_\msc(\vx) = \begin{cases}0 & \text{ if } \vx \in \msc \\ +\infty &\text{ otherwise,}\end{cases}
\end{equation}
with $\rmQ$ a $m\times d$ matrix consisting of $m$ random lines from the
$d \times d$ identity matrix.

\subsection{Image dataset}
In Figures~\ref{fig:dataset-part1} and ~\ref{fig:dataset-part2} we present the 6
original images used in the experiments. These images contain both geometric
structures, constant areas and textured regions. On the same figures, we display
degraded versions of each image for each set of experiments. For the denoising
experiment, the level of the Gaussian noise is fixed to $\sigma^2 =
(30/255)^2$. In the case of deblurring, the operator $\rmA$ correponds to a
$9 \times 9$ uniform blur operator, and we add Gaussian noise with variance
$\sigma^2=(1/255)^2$. Finally, in the context of inpainting, we hide $80\%$ of
the pixels.

\begin{figure}
    \centering

\begin{tabular}{cccc}

& Alley & Bridge & Cameraman \\

\CenteredVcell{0.2\textwidth}{Clean Images} & 
\includegraphics[width=0.2\textwidth]{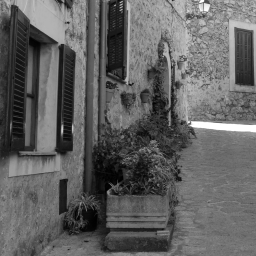} &
\includegraphics[width=0.2\textwidth]{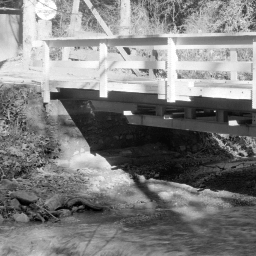} &
\includegraphics[width=0.2\textwidth]{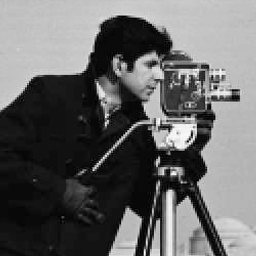} \\

\CenteredVcell{0.2\textwidth}{Denoising} & 
\includegraphics[width=0.2\textwidth]{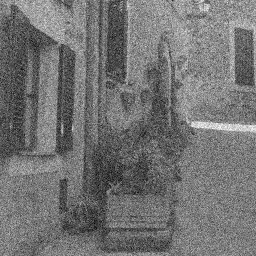} &
\includegraphics[width=0.2\textwidth]{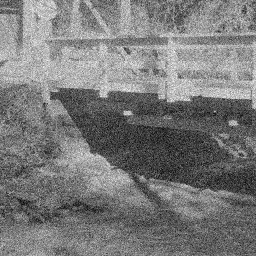} &
\includegraphics[width=0.2\textwidth]{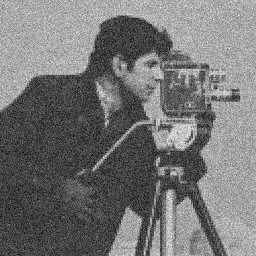} \\

\CenteredVcell{0.2\textwidth}{Deblurring} & 
\includegraphics[width=0.2\textwidth]{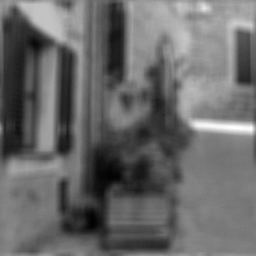} &
\includegraphics[width=0.2\textwidth]{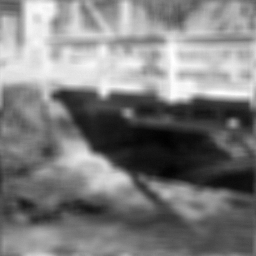} &
\includegraphics[width=0.2\textwidth]{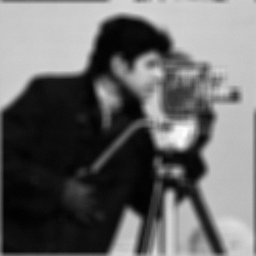} \\

\CenteredVcell{0.2\textwidth}{Inpainting} & 
\includegraphics[width=0.2\textwidth]{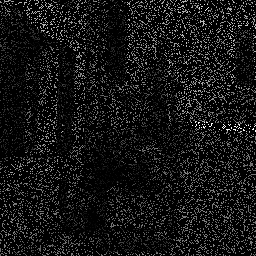} &
\includegraphics[width=0.2\textwidth]{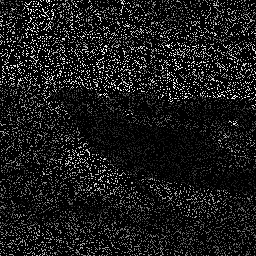} &
\includegraphics[width=0.2\textwidth]{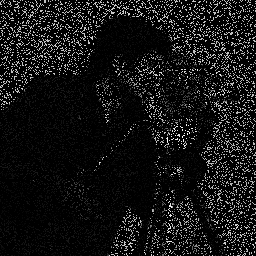}
\end{tabular}

\caption{\emph{Dataset (part 1):} First three images in our dataset, and examples of degraded images for the three inverse problems considered in this paper. For denoising, we add a Gaussian noise with variance $\sigma^2 = (30/255)^2$. For deblurring, the operator $\rmA$ correponds to a $9 \times 9$ uniform blur operator, and we add Gaussian noise with variance $\sigma^2=(1/255)^2$. For inpainting, we hide $80\%$ of the pixels.}
\label{fig:dataset-part1}
\end{figure}

\begin{figure}
    \centering
\begin{tabular}{cccc}

& Goldhill & Simpson & Traffic \\

\CenteredVcell{0.2\textwidth}{Clean Images} & 
\includegraphics[width=0.2\textwidth]{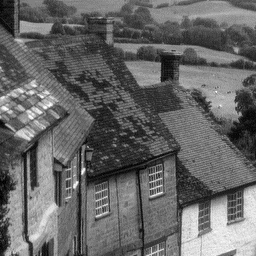} &
\includegraphics[width=0.2\textwidth]{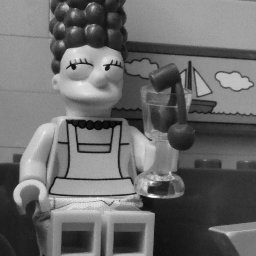} &
\includegraphics[width=0.2\textwidth]{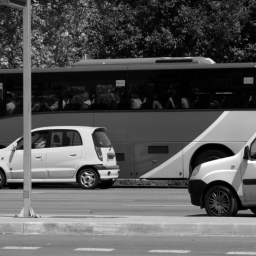} \\

\CenteredVcell{0.2\textwidth}{Denoising} & 
\includegraphics[width=0.2\textwidth]{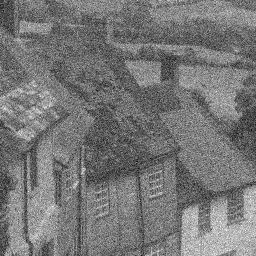} &
\includegraphics[width=0.2\textwidth]{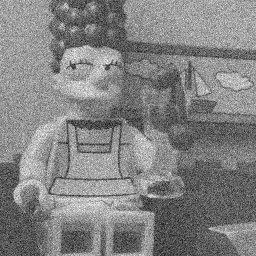} &
\includegraphics[width=0.2\textwidth]{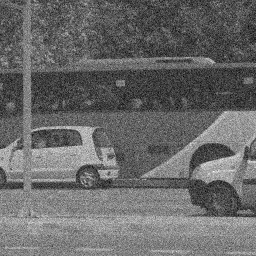} \\

\CenteredVcell{0.2\textwidth}{Deblurring} & 
\includegraphics[width=0.2\textwidth]{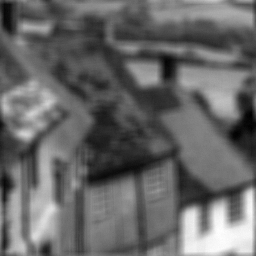} &
\includegraphics[width=0.2\textwidth]{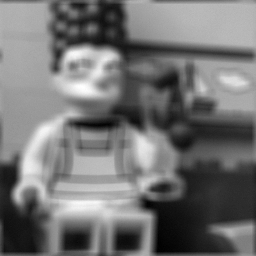} &
\includegraphics[width=0.2\textwidth]{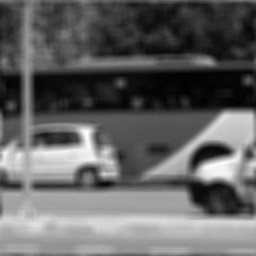} \\

\CenteredVcell{0.2\textwidth}{Inpainting} & 
\includegraphics[width=0.2\textwidth]{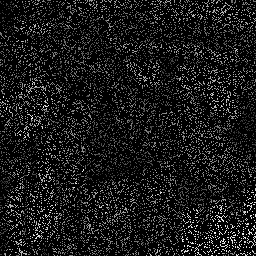} &
\includegraphics[width=0.2\textwidth]{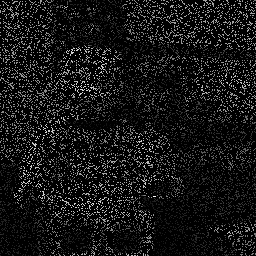} &
\includegraphics[width=0.2\textwidth]{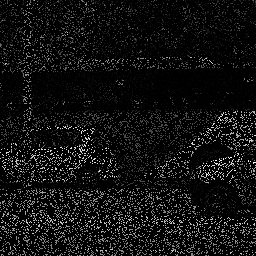}
\end{tabular}
\caption{\emph{Dataset (part 2):} Last three images in our dataset, and examples of degraded images for the three inverse problems considered in this paper. For denoising, we add a Gaussian noise with variance $\sigma^2 = (30/255)^2$. For deblurring, the operator $\rmA$ correponds to a $9 \times 9$ uniform blur operator, and we add Gaussian noise with variance $\sigma^2=(1/255)^2$. For inpainting, we hide $80\%$ of the pixels.}
\label{fig:dataset-part2}
\end{figure}

\subsection{Algorithms}

In this section, we evaluate PnP-SGD (Algorithm~\ref{algo:pnp-sgd}) along with
three other classical PnP algorithms: PnP-ADMM (Algorithm~\ref{algo:pnp-admm}),
PnP-FBS (Algorithm~~\ref{algo:pnp-fbs}) and PnP-BBS
(Algorithm~\ref{algo:pnp-bbs}).  %
Note that in the case of inpainting the log-likelihood is not differentiable,
since $\iota_C$ is not differentiable. In Section~\ref{sec:inpainting} we will
present an extension of these PnP algorithms to this setting using proximal
operators.

In order to take into account the parameter $\alpha > 0$ into
Algorithms~\ref{algo:pnp-admm}-\ref{algo:pnp-fbs}-\ref{algo:pnp-bbs}, we slightly modify the target
function. Instead of minimizing $x \mapsto -\log p(x|y)$ we aim at minimizing
$x \mapsto -\log p(x|y)/\alpha$. Doing so the parameter $\alpha > 0$ can be
included in the parameters of the log-likelihood which becomes
$(x,y) \mapsto F(x,y)/\alpha$.  All algorithms are implemented using Python and
the PyTorch library. Our experiments are run on an Intel Xeon CPU E5-2609 server
with a Nvidia Titan XP graphic card.

\begin{algorithm}
	\caption{PnP-ADMM}\label{algo:pnp-admm}
	\begin{algorithmic}

		\REQUIRE $n \in \nset$, $y \in \rset^\dimY$, $\vareps > 0, \alpha > 0$, $x_0 \in \rset^d$
		
		\STATE \textbf{Initialization:} Set $\vx_0 = \vz_0$, and $u_k = 0$.
		
		\FOR{$k=0:N$}
		
		\STATE $\vx_{k+1} = \prox_{(\varepsilon /\alpha) \Fdata(\cdot,\vy)} (\vz_k - u_k) $
		
		\STATE $\vz_{k+1} = D_{\varepsilon} (\vx_{k+1} + u_k)$ 
		
		\STATE $u_{k+1} = u_k + (\vx_{k+1} - \vz_{k+1})$ 
		
		\ENDFOR
		
		\STATE \textbf{return} $x_{N+1}$
			
	\end{algorithmic}
\end{algorithm}

\begin{algorithm}
	\caption{PnP-FBS}\label{algo:pnp-fbs}
	\begin{algorithmic}

          \REQUIRE $n \in \nset$, $y \in \rset^\dimY$, $\vareps>0$,
          $\alpha > 0$, $x_0 \in \rset^d$	
		
		\FOR{$k=0:N$}
		
		\STATE $\vx_{k+1} = D_{\varepsilon} \left(\vx_k -  (\vareps/\alpha) \nabla \Fdata(\vx_k,\vy)  \right)$
		
		\ENDFOR
		
		\STATE \textbf{return} $x_{N+1}$
			
	\end{algorithmic}
\end{algorithm}

\begin{algorithm}
	\caption{PnP-BBS}\label{algo:pnp-bbs}
	\begin{algorithmic}

          \REQUIRE $n \in \nset$, $y \in \rset^\dimY$, $\vareps>0$, $\alpha>0$,
          $x_0 \in \rset^d$

		\FOR{$k=0:N$}
		
		\STATE $\vx_{k+1} = D_{\varepsilon} (\prox_{(\vareps/\alpha) \Fdata(\cdot,\vy)}(\vx_k) )$
		
		\ENDFOR
		
		\STATE \textbf{return} $x_{N+1}$
	\end{algorithmic}
\end{algorithm}

\subsection{Parameters settings and convergence conditions}
\label{sec:convergence_conditions}

In this section, we recall and discuss the choice of the different parameters,
as well as the convergence conditions for PnP-SGD. We also discuss the
convergence properties of PnP-ADMM and PnP-FBS following the guidelines
of~\cite{ryu2019plug,Xu2020}.

Recall that from \eqref{eq:posterior_eps}, we denote $\Ltt_y$ the Lipschitz
constant of the log-likelihood gradient $x \mapsto \nabla \Fdata(.,\vy)$. For
$\Fdata(\vx,\vy) = \|\rmA \vx - \vy\|^2/(2\sigma^2)$,
$\Ltt_y = \|\rmA^\star\rmA\| /\sigma^2$, with $\rmA^\star$ the adjoint of
$\rmA$.  $\Fdata$ is $\mu$-strongly convex if and only if $\rmA$ is invertible,
in which case $\mu = \lambda_{\min}(\rmA)^2/\sigma^2$, where
$\lambda_{\min}(\rmA)$ is the smallest singular value of $\rmA$. In our
experiments we have $\lambda_{\min}=1$ for denoising and $\lambda_{\min} =0$ for
deblurring and inpainting.  In our experiments, the operator $\rmA$ is always
chosen such that $\|\rmA^\star\rmA\|=1$.  Note that if $\Fdata$ is replaced by
$\Fdata/\alpha$, as it the case in
Algorithms~\ref{algo:pnp-admm}-\ref{algo:pnp-fbs}-\ref{algo:pnp-bbs}, we have
that $\Ltt_y$ and $\mu$ are replaced by $\Ltt_y/\alpha$ and $\mu/\alpha$.

\paragraph{Denoiser.} In all experiments, the denoising operator $D_{\vareps}$
is chosen as the pretrained denoising neural network introduced
in~\cite{ryu2019plug}. This denoiser is trained so that $\Id- D_{\vareps}$ is
$\Ltt$-Lipschitz with $\Ltt < 1$. Note that this corresponds to the first part
of \eqref{eq:cond_neural_net} in H\ref{assum:neural_net}. In \cite{ryu2019plug}
three pretrained denoisers, at noise level
$\vareps = (5/255)^2, (15/255)^2, (40/255)^2$ are proposed. In this work, we
only use the first one in our denoising and deblurring experiments. The
inpainting problem requires a more subtle strategy relying on a coarse to fine
approach, described in Section~\ref{sec:inpainting}.

\paragraph{PnP-SGD.}
In Algorithm~\ref{algo:pnp-sgd}, we consider a burn-in regime with a constant
step $\delta_0$ until some iteration $n_{\mathrm{burnin}}$. After this initial
phase, we set $(\delta_k)_{k \in \nset}$ to be a decreasing sequence satisfying
the conditions of Proposition~\ref{prop:cv_pnpsgd}.  In the case of denoising or
deblurring, $\delta_0$ is given by
\begin{equation}
  \label{eq:deltastable}
  \delta_0 =  \deltaStable/6, \text{ where } \deltaStable:=
  2/\Ltt_{\mathrm{tot}}  \eqsp , \quad \Ltt_{\mathrm{tot}} =  2/(\alpha \Ltt /\varepsilon +
    \|\rmA^*\rmA\|/\sigma^2) \eqsp , 
\end{equation}
where $\Ltt_{\mathrm{tot}}$ is the Lipschitz constant of
$\nabla \log p (. | \vy)$. Note that setting $\delta_0 = \deltaStable$ ensures
that the deterministic scheme: $x_0 \in \rset^d$ and for any $k \in \nset$,
$x_{k+1} = x_k + \delta_0 \nabla \log p (\vx_k | \vy)$, satisfies that
$(\log p(x_k |y))_{k \in \nset}$ is non-decreasing. %
After the burn-in, we use a decreasing sequence of step-sizes
$(\delta_k)_{k \in \nset}$ such that for any $k \in \nset$ we 
have \begin{equation}\delta_k :=\delta_0 \times (k-\nburnin)^{-0.8} \eqsp , 
\label{eq:deltak}\end{equation}
which satisfies the conditions required in Proposition~\ref{prop:cv_pnpsgd} for
convergence. Note that contrary to existing work, any value of $\alpha > 0$ can
be used in Algorithm~\ref{algo:pnp-sgd} provided that $\delta_0$ is defined
accordingly using \eqref{eq:deltastable}.

\paragraph{PnP-ADMM.}

The convergence results of ~\cite{ryu2019plug} for PnP-ADMM require the strong
convexity of $\Fdata$. In our experiments, this condition is met for denoising
experiments (since $\rmA = \Id$), but not for inpainting nor deblurring if the
blur operator is not invertible (which is the case for a $9\times 9$ uniform
blur).  In the denoising case, following~\cite{ryu2019plug}, PnP-ADMM converges
to a fixed point if $\Ltt \in [0,1)$ and
$ \Ltt /(1+\Ltt(1-2\Ltt)) < \vareps/(\alpha \sigma^2)$.  In practice, this
condition is not satisfied, see Section~\ref{sec:denoising}. However,
Algorithm~\ref{algo:pnp-sgd} experimentally converges  to a fixed
point with interesting visual properties. This suggests that it might be
possible to prove the convergence of PnP-ADMM under weaker conditions than the
ones of \cite{ryu2019plug}.

\paragraph{PnP-FBS.} Similarly to PnP-ADMM the convergence results obtained
by~\cite{ryu2019plug} for PnP-FBS are only valid in a strongly convex
setting. In our case this corresponds to the denoising experiment here. The
condition on the Lipschitz constant of the denoiser $D_\vareps$ is
$\Ltt / (1+\Ltt) < \vareps / (\alpha \sigma^2) < (\Ltt +2)/(\Ltt+1)$.  In
Section \ref{sec:denoising}, we show that these conditions are not met in our
experiments.  In practice, we still observe convergence of the algorithm for the
denoising experiments. This is no longer case in non-strongly convex problems,
see Section~\ref{sec:deblurring} and
Section~\ref{sec:inpainting}. In~\cite{Xu2020}, convergence towards the set of
stationnary points of the log-posterior is established for PnP-FBS provided that
$D_{\vareps} = D_\vareps^\star$, \ie \ $D_\vareps$ is the optimal MMSE. In
addition, \cite{Xu2020} requires that $\vareps \Ltt_{y} \leq 1$. This condition
implies that $\vareps \|\rmA^\star\rmA\|\leq \alpha \sigma^2$. Since
$\|\rmA^\star\rmA\| = 1$ for all our experiments, this implies
$\alpha \geq \vareps / \sigma^2$. In experiments with large noise level (as it
it the case for our denoising setting), this leads to acceptable values of
$\alpha$. However, when $\sigma$ is small in comparison to $\vareps$ (which is
the case for deblurring), the regularisation parameter $\alpha$ for which the
convergence is ensured is too highlighted in
Section~\ref{sec:convergence_conditions}.

\subsection{Denoising}
\label{sec:denoising}
For these denoising experiments, we add a Gaussian noise of variance
$\sigma^2 = (30/255)^2$ (see the second row of Figures~\ref{fig:dataset-part1}
and~\ref{fig:dataset-part2} for examples of degraded images). In this experiment
we use a denoiser $D_{\vareps}$ trained for a noise level $\vareps = (5/255)^2$
on a dataset $\{x_i,x_i'\}_{i=1}^N$ with $x_i \sim p$ and
$x_i' \sim \mathcal{N}(x_i, \vareps \Id)$ for any $i \in \{1, \dots, N\}$. Using
this denoiser in Algorithms~\ref{algo:pnp-sgd}-\ref{algo:pnp-bbs}, we aim at
denoising $y$ with noise level $\sigma^2$.

We run all algorithms for several values of the regularization parameter
$\alpha$ and for two different initializations: first a
$\mathrm{TV}$-$\mathrm{L}_2$ initialization, \ie \ applying a simple
$\mathrm{TV}$-$\mathrm{L}_2$ restoration to the noisy image following~\cite{Rudin1992,chambolle2011first}, and second an oracle initialization (using the original
image without degradation). Our goal here is to assess the dependency of the
algorithm on initialization, since the log-posterior we study is highly
non-convex. %

For PnP-SGD, the initial step-size $\delta_0$ and the sequence
$(\delta_k)_{k \in \nset}$ are defined as explained in
Section~\ref{sec:convergence_conditions}. For these denoising experiments, the
resulting value of $\delta_0$ is already quite small, such that decreasing
$\delta_k$ after the burn-in phase effectively stops the search for a better
optimum and does not change the result.  The number of iterations \nburnin\ for
the burn-in was set between 5000 and 25000 for SGD.  Within that range, we stop
this phase as soon as
$|\mathrm{PSNR}(X_{k+1})-\mathrm{PSNR}(X_k)|< 0.1 \times \delta_0$. This
conservative choice allows to make sure that the algorithm reaches its steady
state, so that the oracle initialization (starting from an overestimated value of $\mathrm{PSNR}$) does not
overestimate the global maximum and the non-oracle initializations (starting from an underestimated value of 
$\mathrm{PSNR}$) do not under-estimate it. In practice,
convergence is reached after a few hundreds of iterations in most cases and only
rarely did the algorithm iterate beyond 5000. Increasing $\delta_0$ to
$\delta_0 = 0.9 \times \deltaStable$ also permits to achieve faster convergence,
but in this case adding a decreasing phase for $(\delta_k)_{k \in \nset}$ after
the burn-in regime is important to achieve the same asymptotic results.

For the splitting-based algorithms (ADMM, BBS, FBS), practical convergence is
very fast and 100 iterations are largely sufficient in all cases. Observe that
since we use a denoiser trained for a noise level $\vareps = (5/255)^2$, and our
denoising experiments are run for $\sigma^2=(30/255)^2$, theoretical convergence
of PnP-ADMM following~\cite{ryu2019plug} requires that
$\alpha < (1+\Ltt(1-2\Ltt)) /36 \Ltt$. The exact value of $\Ltt$ for the
denoising considered in \cite{ryu2019plug} is not available, but our experiments
suggest that $\Ltt \approx 1$. This implies that only drastically small values
of $\alpha$ meet the previous condition. As a result, this condition is not
satisfied with the choices of $\alpha$ that are experimentally optimal but does
not prevent the algorithm to converge in practice.  In the same way, provided
that $\Ltt \in [0,1)$, convergence of PnP-FBS following~\cite{ryu2019plug}
implies that $\alpha$ is at least larger than $18$, see
Section~\ref{sec:convergence_conditions}. Yet, interesting values of $\alpha$
for this denoising experiment are far smaller The condition provided
in~\cite{Xu2020}, $\alpha \geq \vareps / \sigma^2 = 1/3$ gives more realistic
values for $\alpha$ but we remind that in this case we must assume that
$D_\vareps = D_\vareps^\star$.

Figure~\ref{fig:denoising-graphs} summarizes the results of this denoising
experiment on 10 independent random noise realizations on each of the 6 images
in the dataset, for PnP-SGD, PnP-ADMM and PnP-BBS (PnP-FBS is not shown here for
the sake of clarity, but it shows a very similar behavior). We first observe
that initialization seems to play a very minor role for all the algorithms
considered in this problem. A $\mathrm{TV}$-$\mathrm{L}_2$ initialization is
sufficient to reach virtually the same reconstruction quality as the oracle
initialization. This might be explained by the fact that denoising is a
relatively simple inverse problem.  Second, all algorithms produce very similar
results, with an optimal value of $\alpha$ around $0.25$, see Figure~\ref{fig:denoising-graphs}. Table~\ref{tab:denoising-TVL2-optimal-alpha} summarizes the denoising
results of all algorithms (including PnP-FBS) obtained for this nearly optimal
setting of $\alpha=0.25$.  In Figure~\ref{fig:denoising-goldhill-k3} we display
the results of the different algorithms for this denoising experiment. If the
$\mathrm{PSNR}$ values are quite close, it seems that the algorithms make
different compromises in terms of visual results. For example, the estimator
obtained with PnP-ADMM seems to exhibit sharper edges. However, it also seems to
hallucinate more false structures than other algorithms.

\begin{figure}
\centering
\hspace*{-1em}\begin{tabular}{c|ccc}
$\mathrm{TV}$-$\mathrm{L}_2$ init & \multicolumn{3}{c}{$\mathrm{TV}$-$\mathrm{L}_2$ vs. oracle init} \\

{SGD vs ADMM vs BBS} & SGD & ADMM & BBS \\

\includegraphics[width=0.23\textwidth]{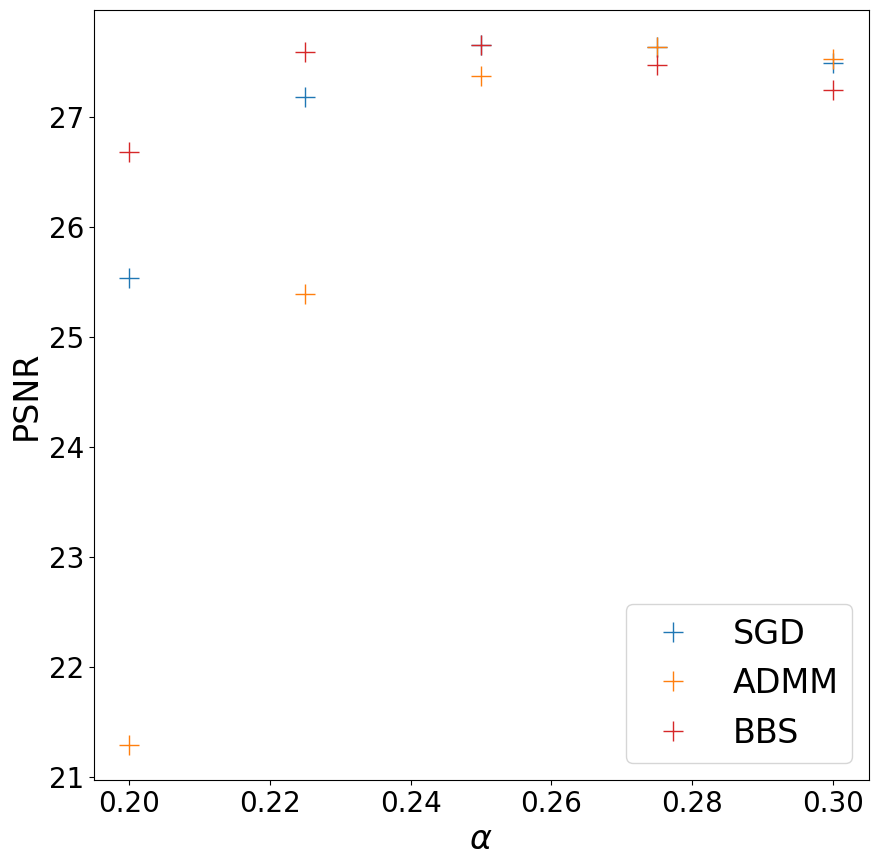} &
\includegraphics[width=0.215\textwidth,trim=40 0 0 0,clip]{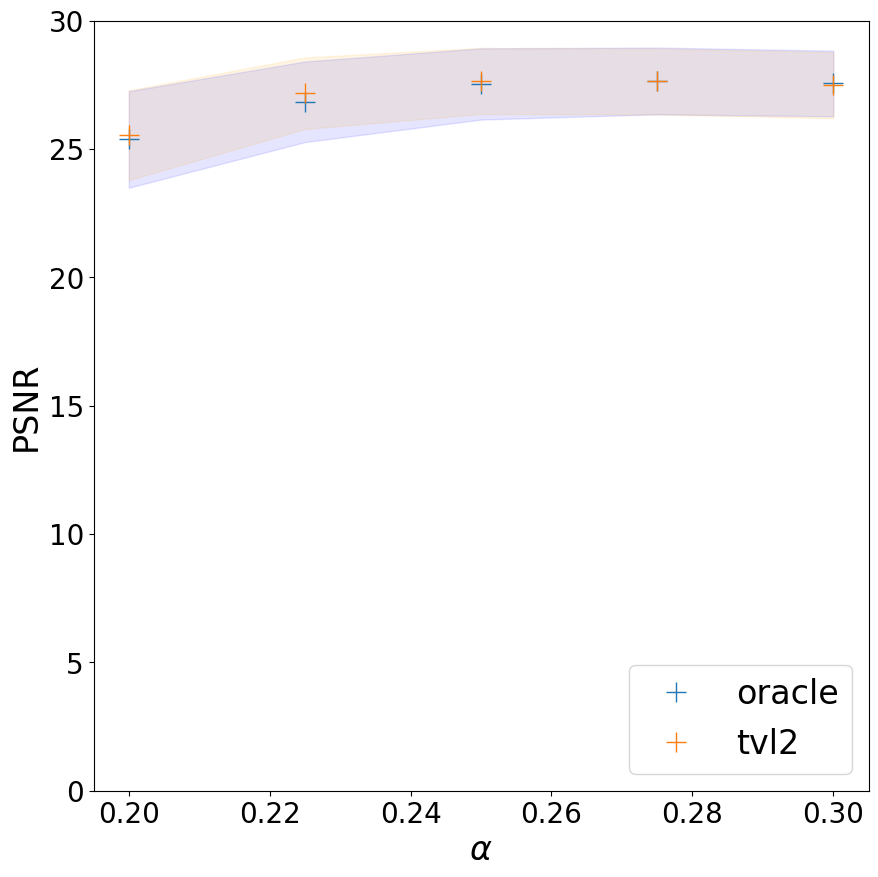} &
\includegraphics[width=0.215\textwidth,trim=40 0 0 0,clip]{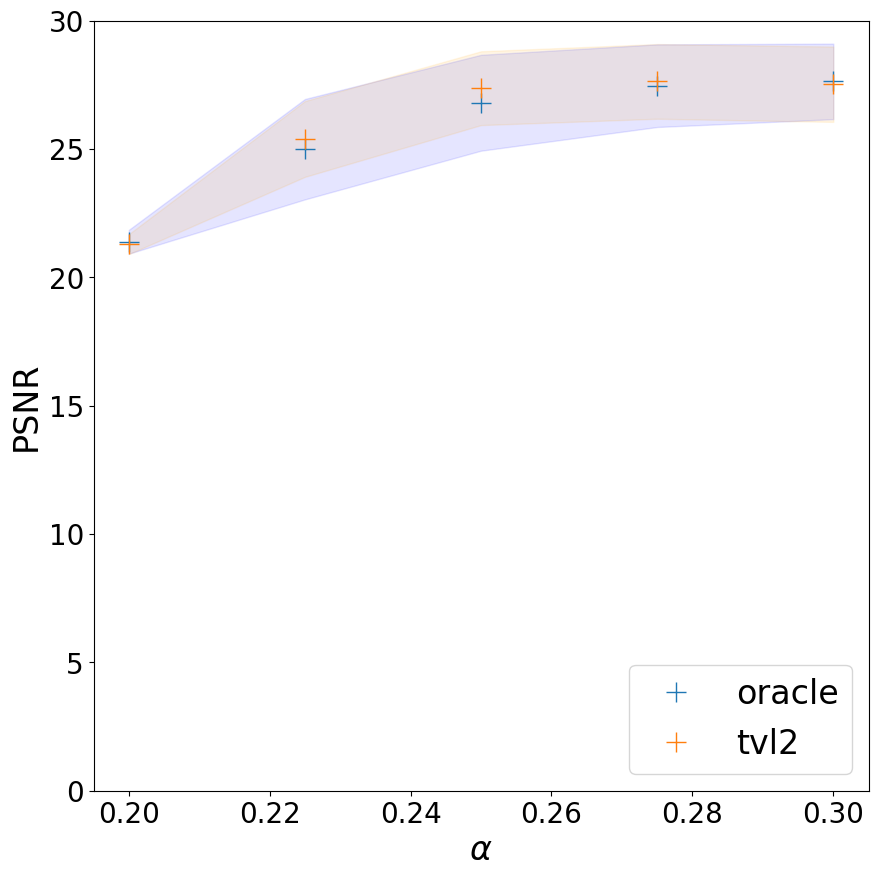} &
\includegraphics[width=0.215\textwidth,trim=40 0 0 0,clip]{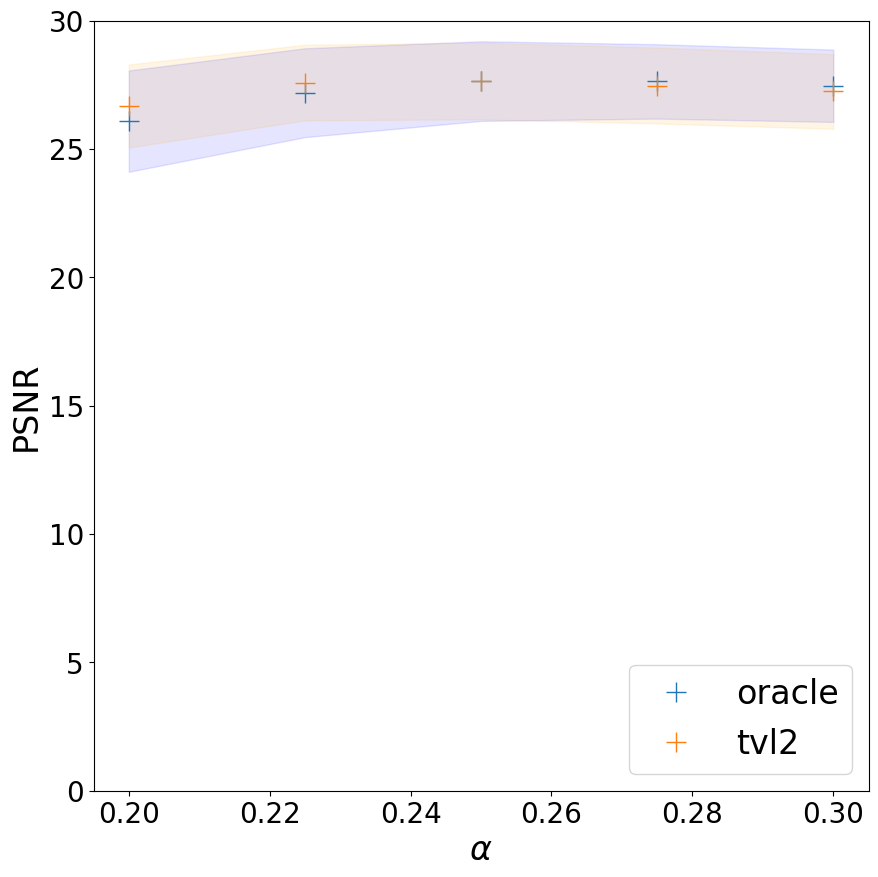} \\

\includegraphics[width=0.23\textwidth]{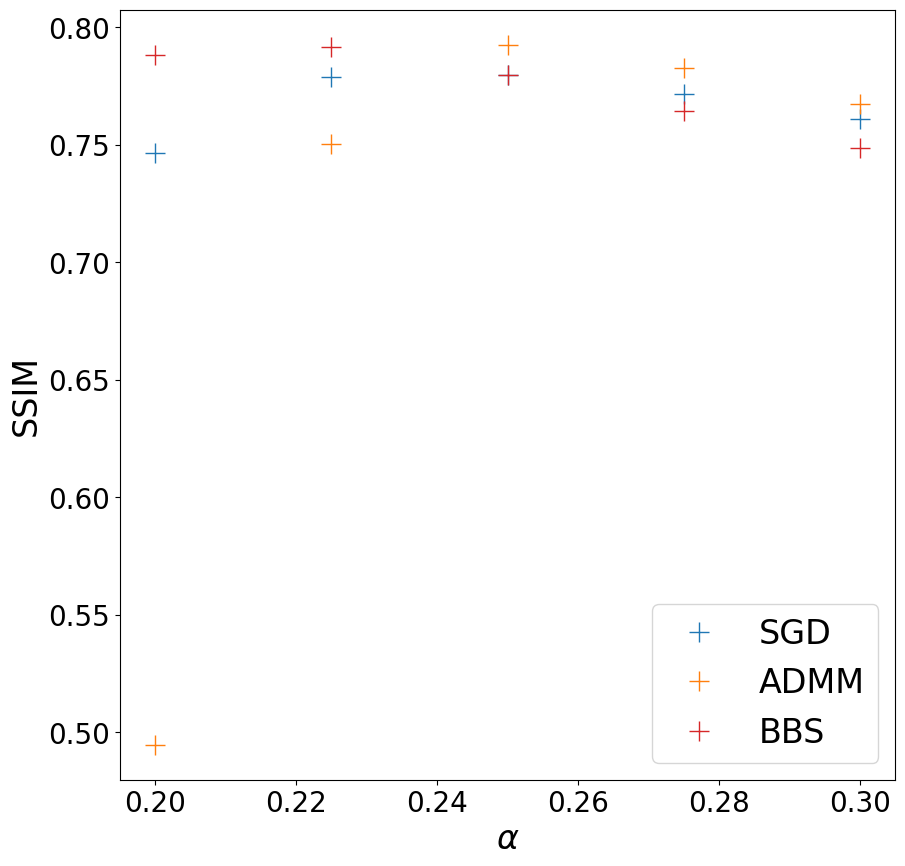} &
\includegraphics[width=0.215\textwidth,trim=40 0 0 0,clip]{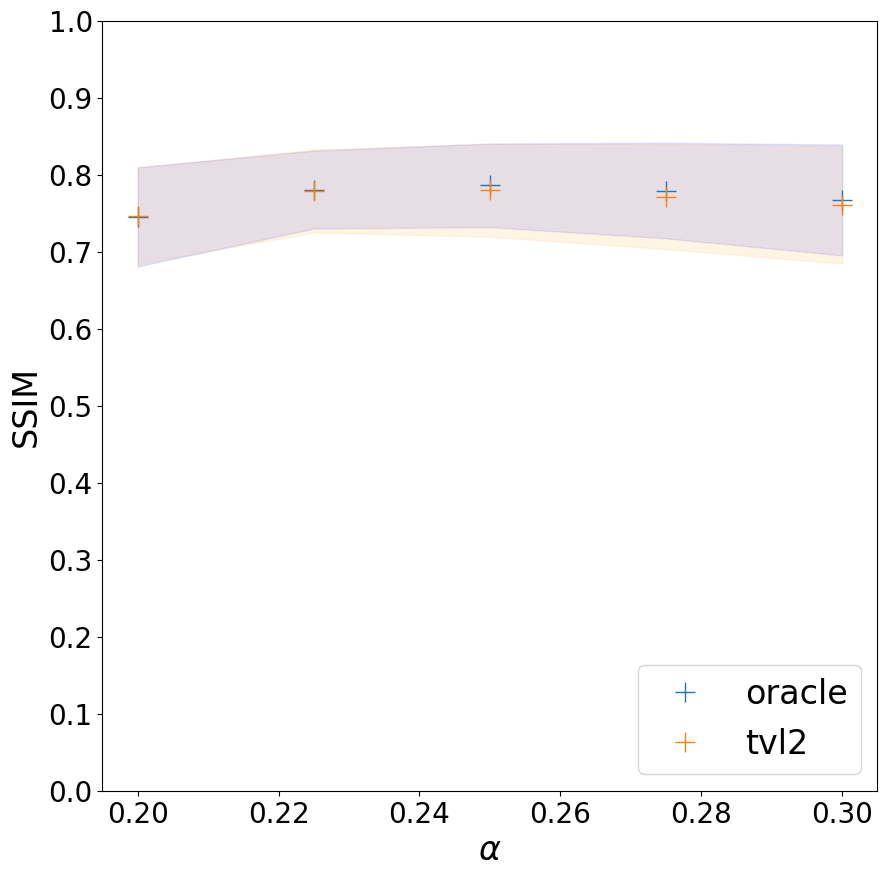} &
\includegraphics[width=0.215\textwidth,trim=40 0 0 0,clip]{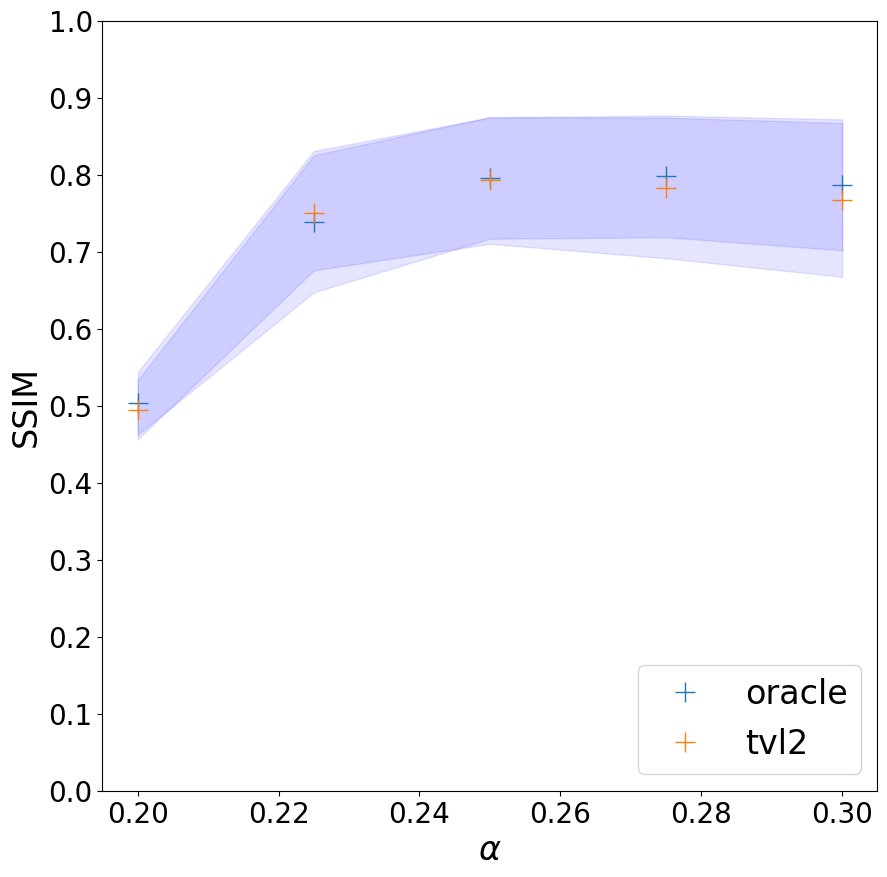} &
\includegraphics[width=0.215\textwidth,trim=40 0 0 0,clip]{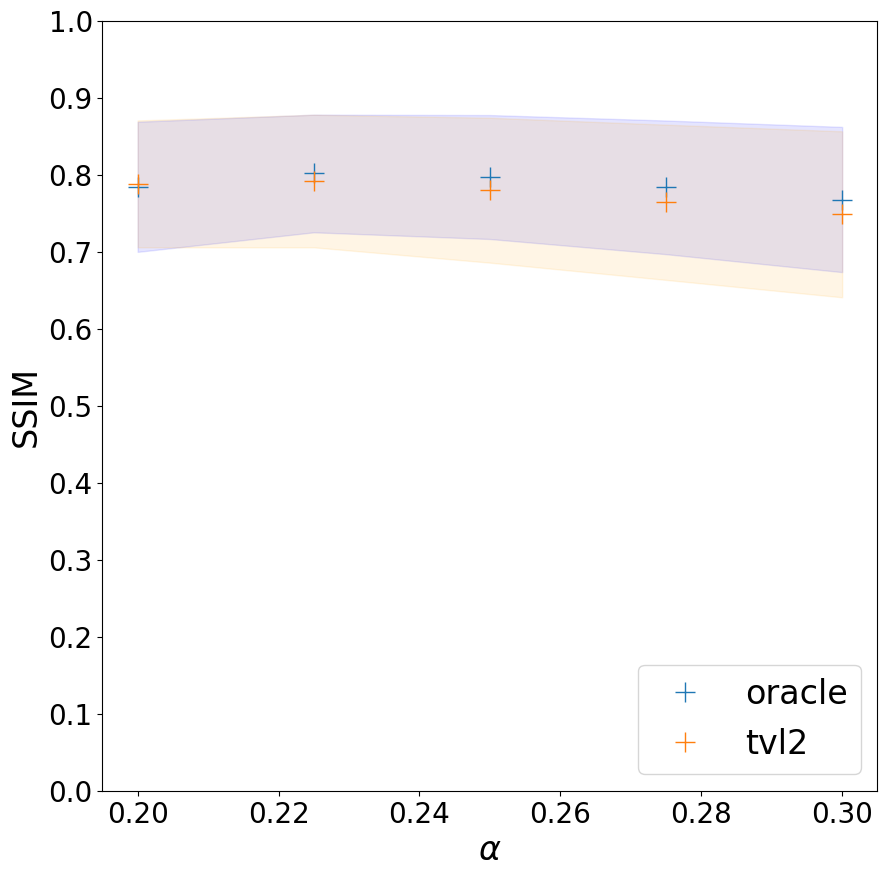} 

\end{tabular}
\caption{Plug \& Play denoising for $\sigma^2=(30/255)^2$ with the prior
  implicit in $D_\vareps$ for $\vareps=(5/255)^2$ and different values of the
  regularization parameter $\alpha$. This table shows means and standard
  deviations for $\mathrm{PSNR}$ and $\mathrm{SSIM}$ values over K=10
  independent noise realizations for each of the six images. Initialization
  plays a very minor role in this case and all algorithms achieve similar
  (nearly optimal) performance for $\alpha=0.25$.}
\label{fig:denoising-graphs}
\end{figure}

\begin{table}\centering
\begin{tabular}{ |l||c|c|c|c|  }
  \hline
  \multicolumn{5}{|c|}{Denoising  $\sigma^2=(30/255)^2$, $\vareps=(5/255)^2$, $\mathrm{TV}$-$\mathrm{L}_2$ init, $\alpha=0.25$} \\
  \hline
  & PnP-SGD  &PnP-ADMM & PnP-BBS & PnP-FBS\\
  \hline
  Overall $\mathrm{PSNR}$   & 27.65 & 27.37 & 27.65 & 27.56  \\
 \hline
 Simpsons      & 30.04 & 30.10 & 30.41 & 30.35 \\
 Traffic & 27.36 & 27.09 & 27.31 & 27.27 \\
Cameraman & 28.54 & 28.21 & 28.74 & 28.48  \\
Alley & 27.16 & 26.82 & 26.98 & 26.96 \\
Bridge & 26.28 & 25.83 & 26.18 & 26.03 \\
Goldhill & 26.55 & 26.18 & 26.30 & 26.30 \\
 \hline
\end{tabular}
\caption{Plug \& Play denoising for $\sigma^2=(30/255)^2$ with the prior in
  $D_\vareps$ for $\vareps=(5/255)^2$. This table shows mean $\mathrm{PSNR}$
  values over K=10 independent noise realizations for each of the six
  images. The regularization parameter $\alpha=0.25$ is nearly optimal for all
  algorithms.}
\label{tab:denoising-TVL2-optimal-alpha}
\end{table}

\begin{figure}
    \centering
\begin{tabular}{cc}
SGD ($\mathrm{PSNR}$=26.58 dB, $\mathrm{SSIM}$=0.69) &
ADMM ($\mathrm{PSNR}$=26.17 dB, $\mathrm{SSIM}$=0.68) \\
\includegraphics[width=0.3\textwidth]{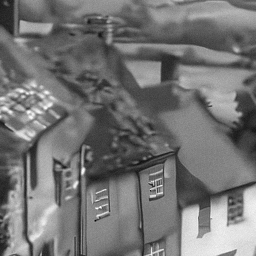}
&
\includegraphics[width=0.3\textwidth]{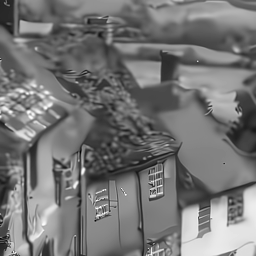}
\\
BBS ($\mathrm{PSNR}$=26.33 dB, $\mathrm{SSIM}$=0.64) &
FBS ($\mathrm{PSNR}$=26.31 dB, $\mathrm{SSIM}$=0.67) \\
\includegraphics[width=0.3\textwidth]{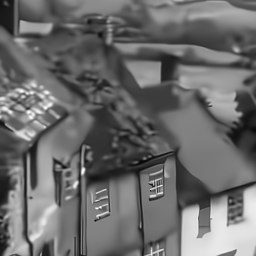}
&
\includegraphics[width=0.3\textwidth]{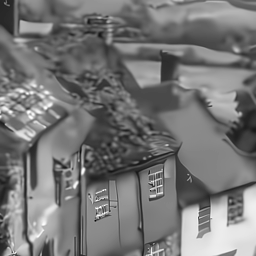}
\end{tabular}

\caption{Plug \& Play denoising for $\sigma^2=(30/255)^2$, $\vareps=(5/255)^2$
  with
  $\alpha=0.25$.} %
    \label{fig:denoising-goldhill-k3}
\end{figure}

\subsection{Deblurring}
\label{sec:deblurring}

We now turn to the deblurring problem.  In this section, images are blurred with
a uniform $9\times 9 $ kernel, and a small Gaussian noise of standard deviation
$\sigma = 1/255$ is added in order to define the degradation model. We now
compare the behavior of Algorithms~\ref{algo:pnp-sgd}-\ref{algo:pnp-bbs}.

Experiments with PnP-SGD follow the same rules as for the denoising problem and
the same observations are valid. When running PnP-ADMM we use approximately 200
iterations to ensure the convergence whereas for PnP-FBS and PnP-BBS, we use
approximately 500 iterations. Except for PnP-SGD (using
Proposition~\ref{prop:cv_pnpsgd}), these PnP algorithms are not guaranteed to
converge according to~\cite{ryu2019plug} since $\rmA$ is not invertible. In
practice PnP-FBS indeed converges only for very large values of the
regularization parameter $\alpha$, whereas other PnP algorithms converge for all
our experiments. As highlighted in Section~\ref{sec:convergence_conditions} this
suggests that convergence for PnP-ADMM and PnP-FBS occur under weaker conditions
than the ones prescribed in \cite{ryu2019plug}.

Figure~\ref{fig:deblurring-graphs} summarizes the results of deblurring on 10
independent random noise realizations on each of the 6 images in the dataset,
for PnP-SGD, PnP-ADMM and PnP-BBS (PnP-FBS is not shown here because it does not
converge most of the time), for $\mathrm{TV}$-$\mathrm{L}_2$ and oracle
initializations. Again, initialization appears to play a minor role in the final
results.

Observe that all algorithms show very similar performances (when they converge)
for these deblurring experiments. While PnP-SGD is slower to converge, it is
ensured to approximate the MAP theoretically.
Table~\ref{tab:deblurring-TVL2-optimal-alpha} summarizes the deblurring results
of all algorithms (including PnP-FBS) obtained for the nearly optimal setting of
$\alpha=0.3$.  In Figure~\ref{fig:denoising-goldhill-k3} we display the results
of the different algorithms for this denoising experiment. Interestingly, we
note that visual results for this deblurring problem are much more similar to
each other than for denoising experiments.

\begin{figure}
\centering
\hspace*{-1em}\begin{tabular}{c|ccc}
$\mathrm{TV}$-$\mathrm{L}_2$ init & \multicolumn{3}{c}{$\mathrm{TV}$-$\mathrm{L}_2$ vs. oracle init} \\

{SGD vs ADMM vs BBS} & SGD & ADMM & BBS \\

\includegraphics[width=0.23\textwidth]{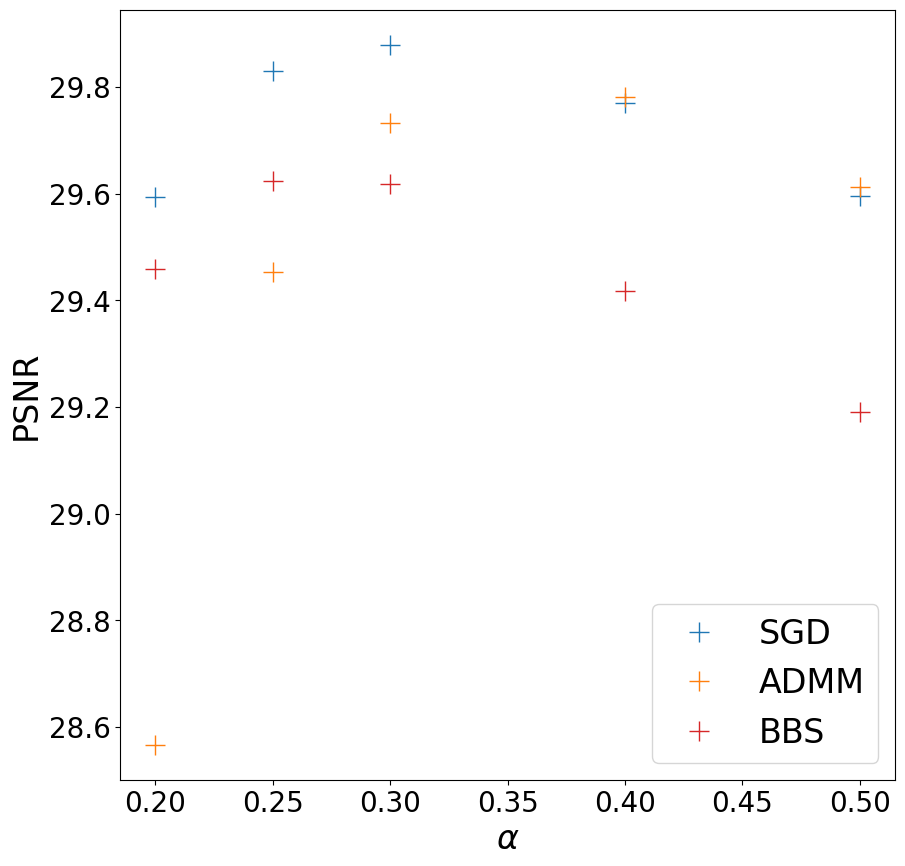} &

\includegraphics[width=0.215\textwidth,trim=40 0 0 0,clip]{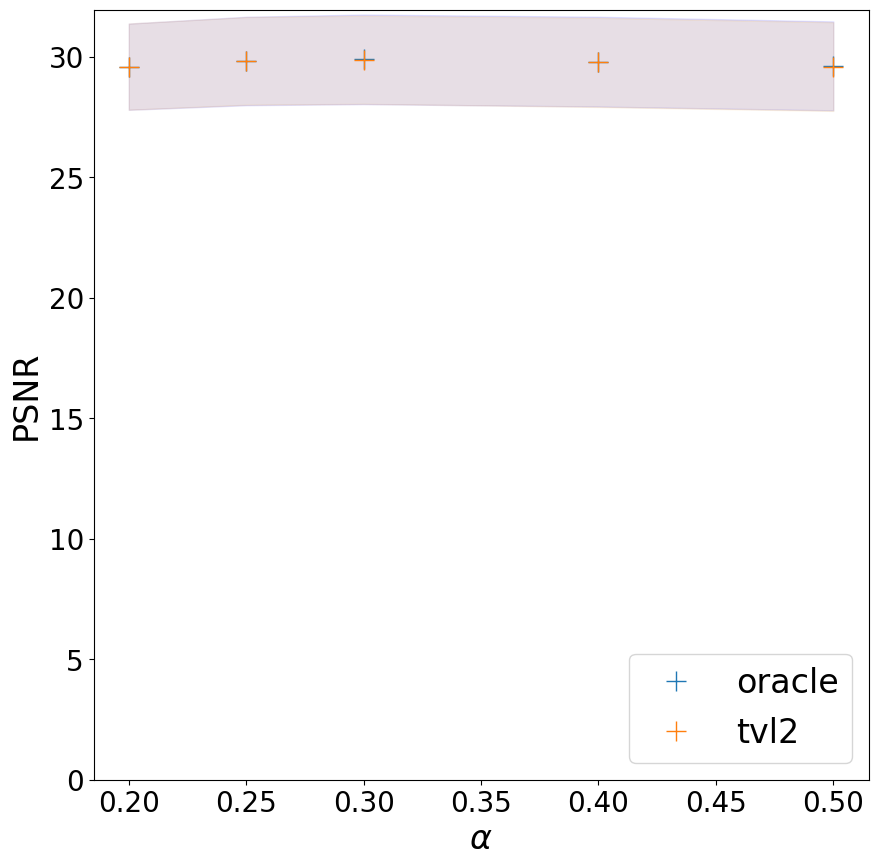} & 
\includegraphics[width=0.215\textwidth,trim=40 0 0 0,clip]{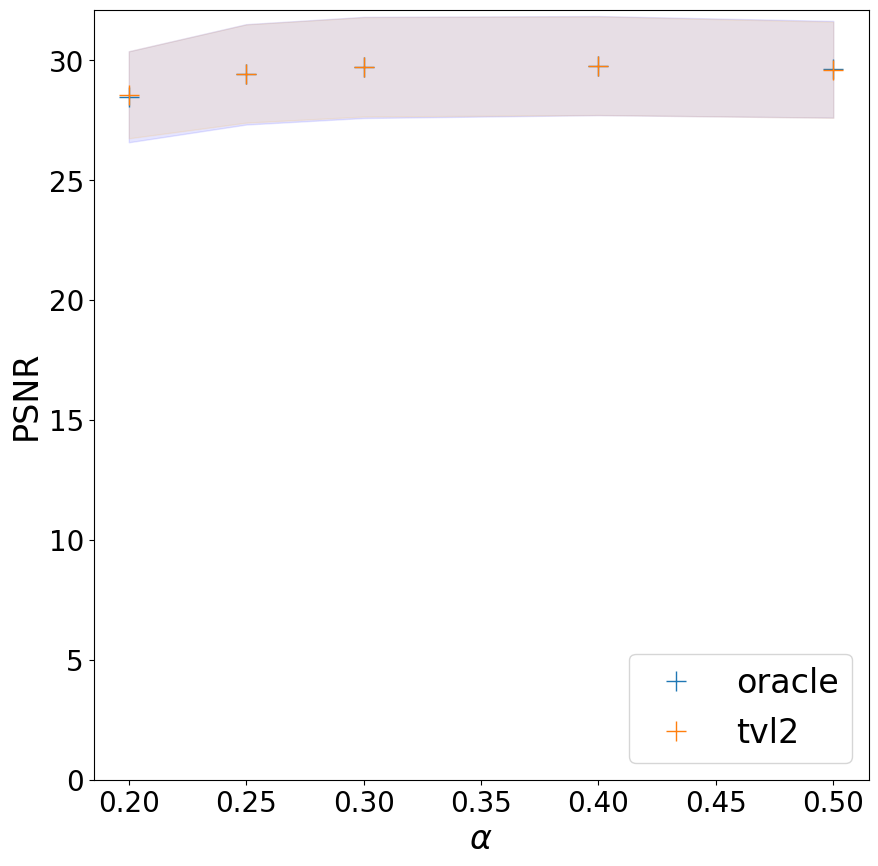} &
\includegraphics[width=0.215\textwidth,trim=40 0 0 0,clip]{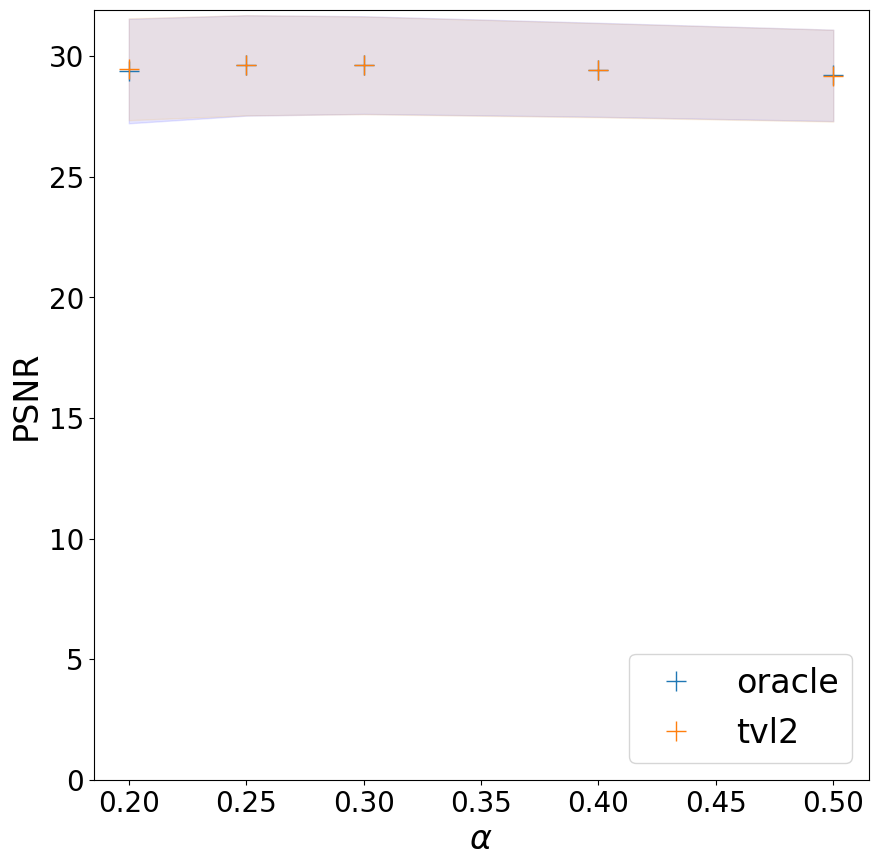} \\

\includegraphics[width=0.23\textwidth]{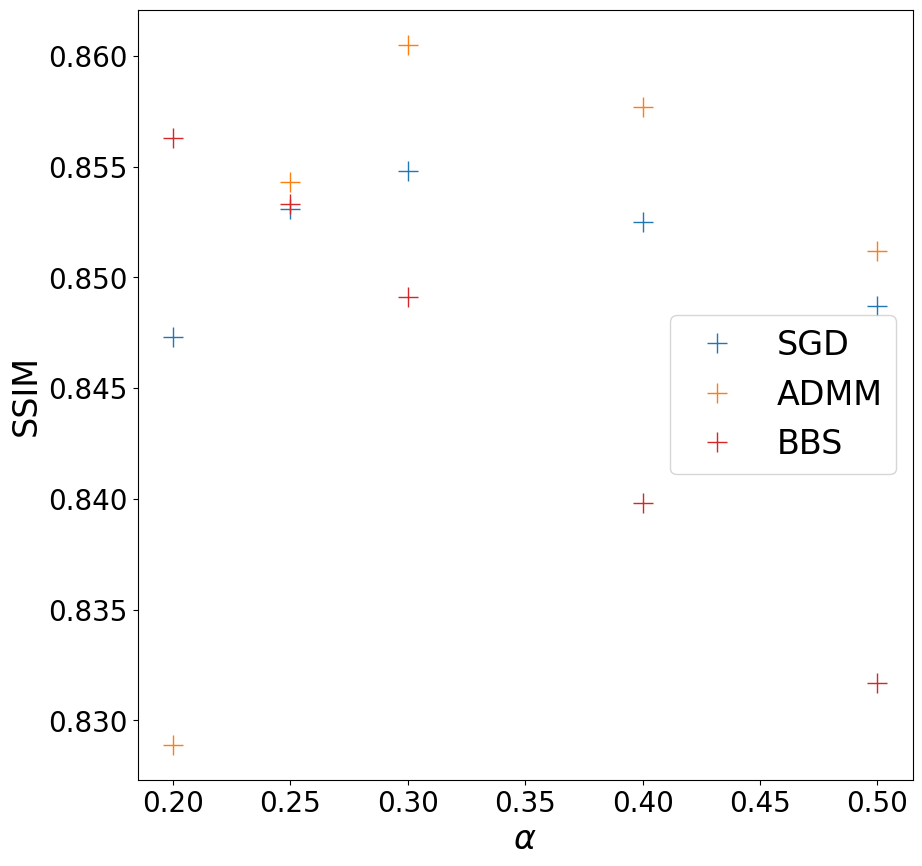} & 

\includegraphics[width=0.215\textwidth,trim=40 0 0 0,clip]{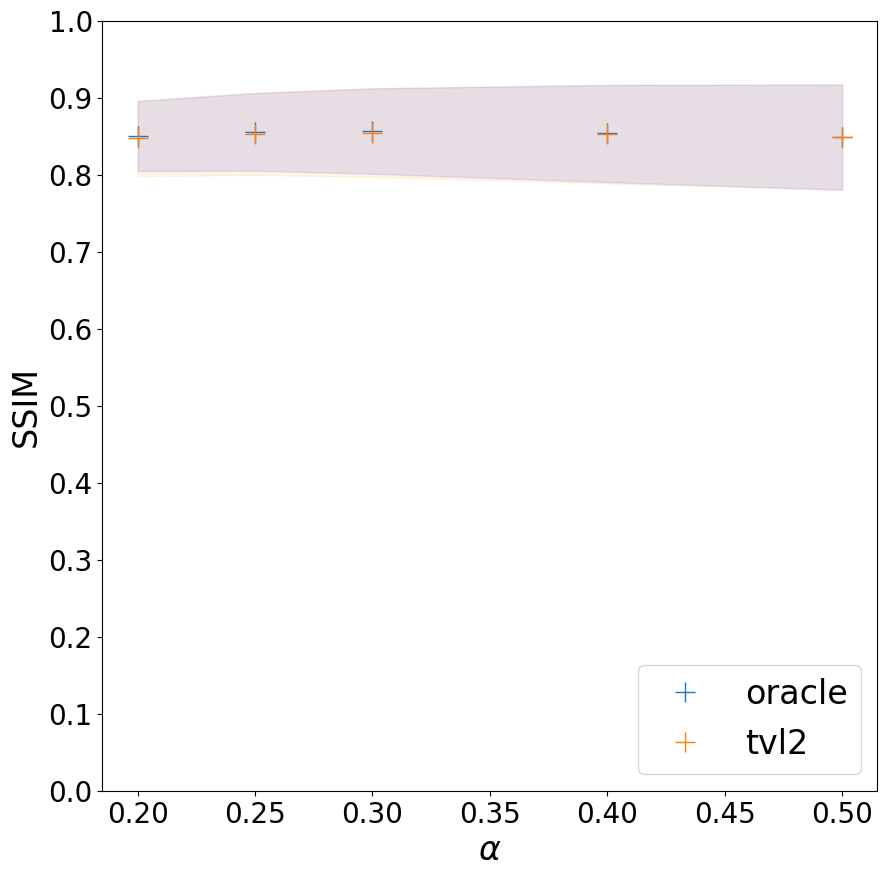} &
\includegraphics[width=0.215\textwidth,trim=40 0 0 0,clip]{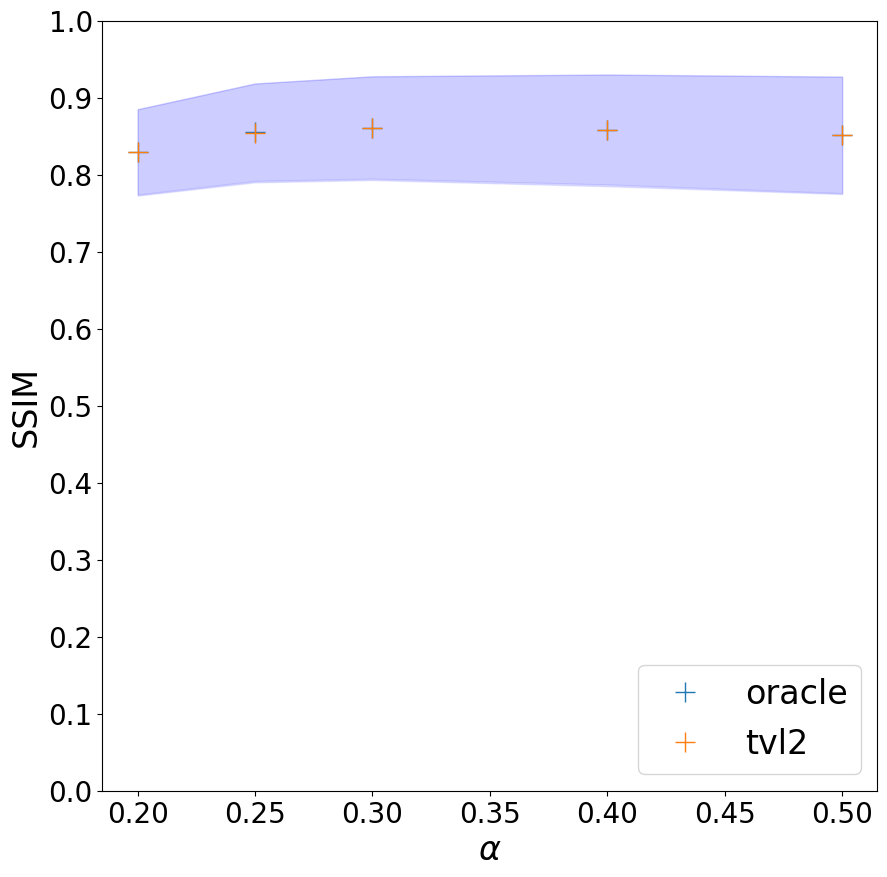} &
\includegraphics[width=0.215\textwidth,trim=40 0 0 0,clip]{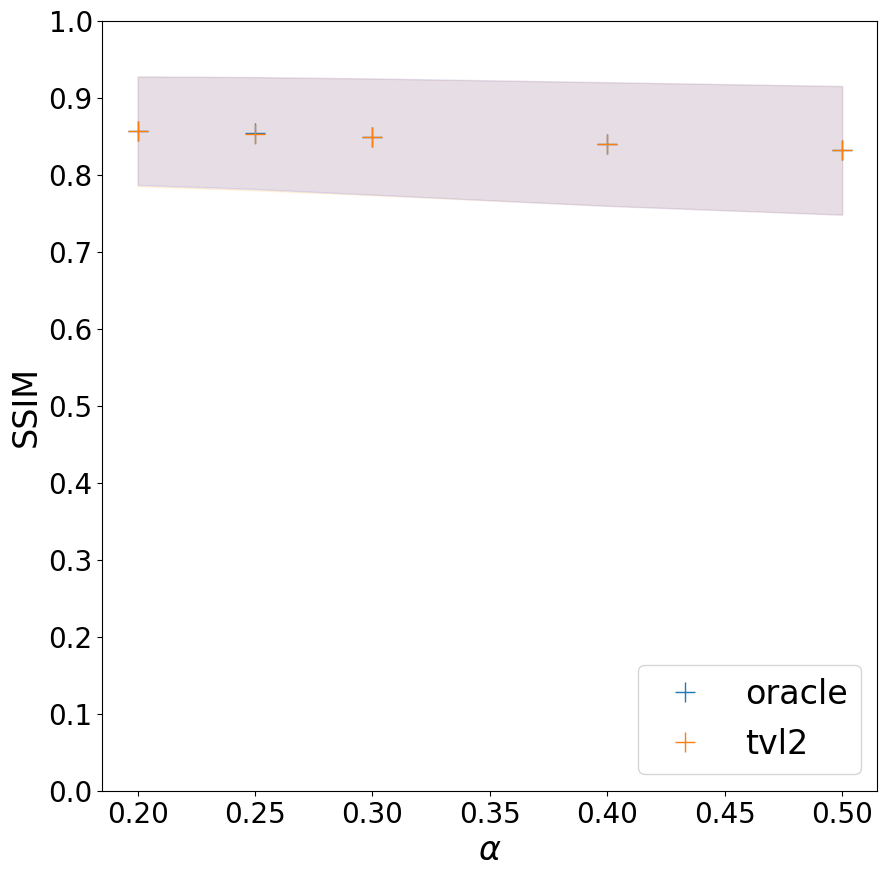}

\end{tabular}

\caption{Plug \& Play deblurring. Image are blurred with a $9\times 9$ uniform kernel, a Gaussian noise of standard deviation $\sigma^2=(1/255)^2$ is added. The denoiser $D_\vareps$ is trained at  $\vareps=(5/255)^2$. The plots shows mean and standard deviation values of $\mathrm{PSNR}$ and $\mathrm{SSIM}$ over K=10 independent noise realizations for each of the six images and different values of the regularization parameter $\alpha$. Initialization plays a very minor role in this case and all algorithms achieve similar (nearly optimal) performance for $\alpha=0.3$, except for FBS which requires a larger (sub-optimal) $\alpha$ to converge.
}
\label{fig:deblurring-graphs}
\end{figure}

\begin{table}\centering
\begin{tabular}{ |l||c|c|c|c|  }
  \hline
  \multicolumn{5}{|c|}{Deblurring a $9\times 9$ kernel with $\sigma^2=(30/255)^2$, $\vareps=(5/255)^2$, $\mathrm{TV}$-$\mathrm{L}_2$ init, $\alpha=0.3$} \\
  \hline
  & PnP-SGD  &PnP-ADMM & PnP-BBS & PnP-FBS\\
  \hline
  Overall $\mathrm{PSNR}$   & 29.88 & 29.73 & 29.62 & NaN \\
 \hline
 Simpsons      & 33.51 & 33.93 & 33.70 & NaN \\
 Traffic  & 29.41 & 29.27 & 29.10 & NaN \\
Cameraman & 30.68 & 30.43 & 30.39 & NaN \\
Alley & 29.26 & 28.99 & 28.90 & NaN \\
Bridge & 28.08 & 27.77 & 27.65 & NaN \\
Goldhill & 28.33 & 28.01 & 27.97 & NaN \\
 \hline
\end{tabular}
\caption{Plug \& Play deblurring. Image are blurred with a $9\times 9$ uniform kernel, a Gaussian noise of standard deviation $\sigma=1/255$ is added. The denoiser $D_\vareps$ is trained at  $\vareps^2=(5/255)^2$. This table shows mean $\mathrm{PSNR}$ values over K=10 independent noise realizations for each of the six images. The regularization parameter $\alpha=0.30$ is nearly optimal for all algorithms.}
\label{tab:deblurring-TVL2-optimal-alpha}
\end{table}

\begin{figure}
     \centering
\begin{tabular}{ccc}
{\tiny SGD ($\mathrm{PSNR}$=28.04 dB, $\mathrm{SSIM}$=0.84)} &
{ \tiny ADMM ($\mathrm{PSNR}$ = 27.77 dB, $\mathrm{SSIM}$=0.83)} &
{\tiny BBS ($\mathrm{PSNR}$=27.64 dB, $\mathrm{SSIM}$=0.82)}\\
\includegraphics[width=0.3\textwidth]{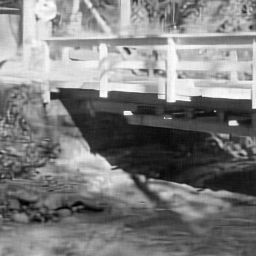}
&
\includegraphics[width=0.3\textwidth]{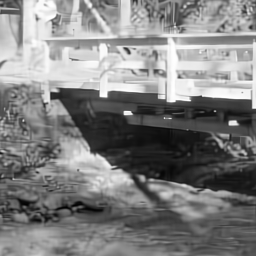}
&
\includegraphics[width=0.3\textwidth]{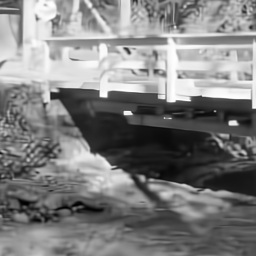}
\end{tabular}
   \caption{Example of Plug \& Play deblurring, for a $9\times 9$ kernel, an additive Gaussian noise of standard deviation $\sigma^2=(1/255)^2$, for $\vareps=(5/255)^2$ and for the nearly optimal value of $\alpha=0.3$.}
     \label{fig:deblurring-bridge-k1}
\end{figure}

\subsection{Inpainting}
\label{sec:inpainting}

The inpainting problem consists in trying to recover $x\in\R^d$ from a small
proportion of its pixels, namely from the measurements vector $y = \rmQ x$,
where $\rmQ$ is a $m \times d$ matrix consisting of $m$ random lines from the
$d \times d$ identity matrix, and $m=qd \ll d$. In our experiments we set
$q=20\%$. In this case, since measurements are not affected by noise, the
data-fitting term takes the form of a hard constraint, \ie \ for any
$x \in \rset^d$ and $y \in \rset^m$ we have
$$F(x,y) = \iota_{\msc_y}(x), \;\text{ where } \msc_y = \left\lbrace x\,:\, y = \rmQ x \right\rbrace \eqsp .$$
The non-differentiability of $F$ is not problem when using ADMM and BBS since in
this case the proximal operator of $\gamma F(\cdot,y)$ is not only defined but
admits a closed-form (which is independent of $\gamma=\vareps/\alpha$). More
precisely, we have for any $x \in \rset^d$ and $y \in \rset^m$,
$\operatorname{prox}_{\gamma\iota_C}(x) = \rmP^\star \rmP x + \rmQ^*y$ in terms
of the $(d-m) \times d$ matrix $\rmP$ containing all the lines of the identity
matrix which are not contained in $\rmQ$.  However, SGD and FBS cannot be
directly applied to this problem because they require $F$ to be
differentiable. Nevertheless we can apply these algorithms to an equivalent
formulation in the reduced space $\R^{d-m}$ of unknown pixels, as shown in the
following subsection.

\subsubsection{Adapting SGD to the non-differentiable inpainting problem}

In what follows, we denote by $\tilde x := \rmP x \in \R^n$ the vector of
$n=d-m$ unknown pixels in $x$. Given the unknown pixels $\tilde x = \rmP x$ and
the measurements $y=\rmQ x$ we can reconstruct $x$ via the affine mapping
$f_y:\R^n \to \R^d$ defined for any $x \in \rset^d$ and $y \in \rset^m$ by
$f_y(\tilde x) = \rmP^* \tilde x + \rmQ^*y $.

The solution of the original problem
$x_{\textsc{map}} = \arg\min_x F(x,y) + U(x)$ can then be written as
\begin{equation}
  \textstyle{
  x_{\textsc{map}} = \arg\min_{x\in \msc_y} U(x) = f_y (\, \arg\min_{\tilde x} U(f_y(\tilde x))) \eqsp , \quad \tilde{x}_{\textsc{map}} = \arg\min_{\tilde x} U(f_y(\tilde x)) \eqsp , }
\end{equation}
and $\tilde{x}_{\textsc{map}}$ can be found by gradient descent on
$\tilde{U} = U\circ f_y$. Using the chain rule and Tweedie's formula, we have
that the gradient of $\tilde{U}$ is given for any $x \in \rset^d$ and
$y \in \rset^m$ by
\begin{equation}
  \nabla \tilde U (\tilde x) = \rmP \nabla U (f_y (\tilde x)) = (1/\vareps) \rmP  (\Id - D_\vareps) \circ f_y(\tilde x) \eqsp . 
\end{equation}
Finally, since the affine operators $\rmP$ and $f_y$ are 1-Lipschitz we have
that $ \tilde{\Ltt} \leq (1/\vareps)$, where $\tilde{\Ltt}$ is the Lipschitz
constant of $\nabla \tilde{U}$.

\subsubsection{Parameter settings and results}

The inpainting problem we consider is extremely ill-posed since 80\% of the
pixels are only constrained by the image prior. Since our implicit prior
$p_\vareps(x)$ is most likely far from log-concave, the posterior shows a
particularly large number of local optima. For this reason all methods are
extremely sensitive to the initial condition.  The initial conditions used in
the previous experiments may misguide both ADMM and SGD to a wrong local
optimum.

To deal with this more difficult case, we consider a different approach, combining:
  \begin{itemize}
\item A coarse to fine scheme where we start by solving the MAP problem for
  large values of $\vareps$, and then use the result of this coarse MAP as an
  initialization for the next smaller value of $\vareps$. In our experiments we
  used $\vareps=(40/255)^2, (15/255)^2, (5/255)^2$, both for
  ADMM and for SGD;
    \item For each value of $\vareps$, a burn-in phase of 2000 iterations with  $\delta_0 = 2.5\deltaStable$, followed by a phase of $1000$ decreasing steps, as defined in~\eqref{eq:deltak}. 
    \end{itemize}
Table~\ref{tab:inpainting-overall} summarizes the results of different algorithmic strategies to solve our inpainting problem, on our set of 6 images with $K=4$ random realizations for each image, and
Figure~\ref{fig:inpainting-simpsons} shows an example of results on the \textit{Simpsons} image.

We can observe in Table~\ref{tab:inpainting-overall} that the coarse-to-fine
scheme is beneficial to both SGD and ADMM, allowing to reach a reconstruction
quality which comes very close to the oracle initialization.
This benefit is also clear on the visual results shown on Figure~\ref{fig:inpainting-simpsons}. In the case of a random initialization, the coarse to fine strategy is needed to avoid the apparition of spurious geometric structure in the background. In the case of the $\mathrm{TV-L_2}$ initialization, it yields better continuity in the fine black lines of the image. This holds both for ADMM and SGD.

In these inpainting experiments, we also observed that using larger initial
step-sizes at the beginning and using the stochastic gradient descent instead of
a simple gradient descent are important to obtain good MAP estimates. This could
be explained by the highly non-convex nature of this problem: the stochastic
term and the larger step sizes are required to avoid getting trapped in spurious
local optima.

\begin{table}
\centering
\begin{tabular}{|l|c|c|c|c|}
\hline
& \multicolumn{2}{|c|}{$\mathrm{PSNR}$} & \multicolumn{2}{|c|}{$\mathrm{SSIM}$} \\
\cline{2-5}
Method & mean & std dev & mean & std dev \\
\hline
\multicolumn{5}{|c|}{Random initialization} \\
\hline
SGD $\vareps=(5/255)^2$           &  23.43 &  2.75 & 0.7715 & 0.0517 \\
SGD $\vareps=(40/255)^2,\,(15/255)^2,\,(5/255)^2$ &\bf 26.32 &  1.76 & 0.8074 & 0.0702 \\
ADMM $\vareps=(5/255)^2$          &  19.34 &  3.09 & 0.6787 & 0.0629 \\
ADMM $\vareps=(40/255)^2,\,(15/255)^2,\,(5/255)^2$&  25.94 &  2.19 & \bf 0.8292 & 0.0745 \\
\hline
\multicolumn{5}{|c|}{$\mathrm{TV}$-$\mathrm{L}_2$ initialization} \\
\hline
SGD $\vareps=(5/255)^2$           &  26.01 & 1.53 & 0.8042 & 0.0684 \\
SGD $\vareps=(40/255)^2,\,(15/255)^2,\,(5/255)^2$ &  \bf 26.34 & 1.80 & 0.8074 & 0.0699 \\
ADMM $\vareps=(5/255)^2$          &  25.38 & 1.74 & 0.8216 & 0.0754 \\
ADMM $\vareps=(40/255)^2,\,(15/255)^2,\,(5/255)^2$&  25.87 & 2.13 & \bf 0.8266 & 0.0764 \\
\hline
\multicolumn{5}{|c|}{Oracle initialization} \\
\hline
SGD $\vareps=(5/255)^2$           & \bf 26.67 & 1.66 & 0.8116 & 0.0700 \\
SGD $\vareps=(40/255)^2,\,(15/255)^2,\,(5/255)^2$ & 26.36 & 1.76 & 0.8079 & 0.0702 \\
ADMM $\vareps=(5/255)^2$          & 26.16 & 2.18 & \bf 0.8330 & 0.0742 \\
ADMM $\vareps=(40/255)^2,\,(15/255)^2,\,(5/255)^2$& 25.93 & 2.14 & 0.8269 & 0.0768 \\
\hline
\end{tabular}
\caption{Inpainting with $p=0.8$, $\sigma=0$ with random, $\mathrm{TV}$-$\mathrm{L}_2$ and oracle initialization. Mean and standard deviation of $\mathrm{PSNR}$ and $\mathrm{SSIM}$ measures computed on K=4 random tests for each of the 6 images. Note the effectiveness of the coarse-to-fine scheme with either random or $\mathrm{TV}$-$\mathrm{L}_2$ initialization: Coarse to fine SGD is only 0.33 dB away from the solution obtained with oracle init, which should be quite close to the global optimum. ADMM is only 0.22 dB away from the solution obtained with oracle init.}
\label{tab:inpainting-overall}
\end{table}

\begin{figure}
\centering
\begin{tabular}{cccc}
 & random init & $\mathrm{TV}$-$\mathrm{L}_2$ init & oracle init \\
 \hline
 & 24.23 / 0.87 & 28.82 / 0.91 & 30.32 / 0.92 \\
\CenteredVcell{0.3\textwidth}{SGD $\sqrt{\vareps}=5/255$} & 
\includegraphics[width=0.3\textwidth]{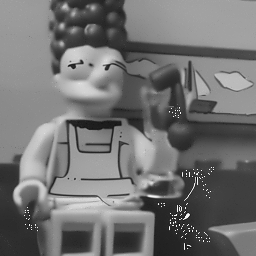} & 
\includegraphics[width=0.3\textwidth]{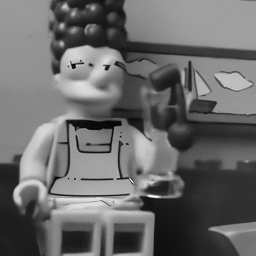} & 
\includegraphics[width=0.3\textwidth]{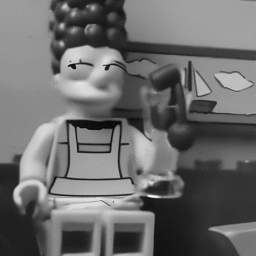}
\\
\hline
& 19.63 / 0.78 & 28.74 / 0.93 & 30.86 / 0.94\\
\CenteredVcell{0.3\textwidth}{ADMM $\sqrt{\vareps}=5/255$} &
\includegraphics[width=0.3\textwidth]{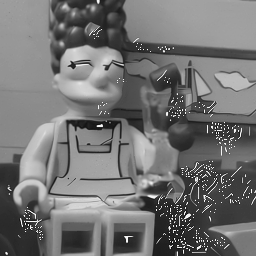} &
\includegraphics[width=0.3\textwidth]{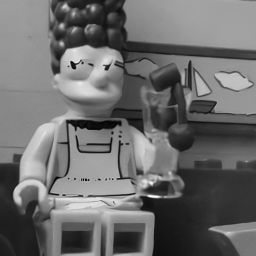} & 
\includegraphics[width=0.3\textwidth]{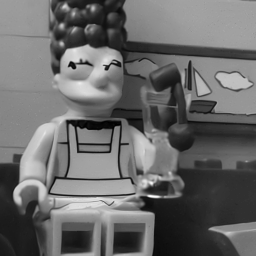}
\\
\hline
& 29.95 / 0.91 & 29.95 / 0.91 &  29.95 / 0.91 \\
\CenteredVcell{0.3\textwidth}{SGD $\sqrt{\vareps}=40/255,\,15/255,\,5/255$} &
\includegraphics[width=0.3\textwidth]{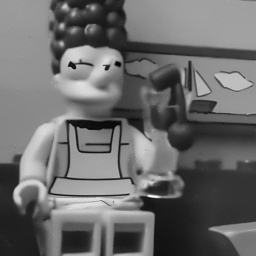} &
\includegraphics[width=0.3\textwidth]{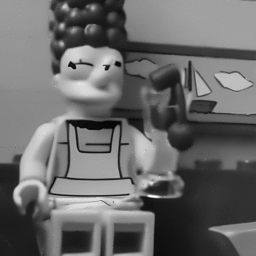} & 
\includegraphics[width=0.3\textwidth]{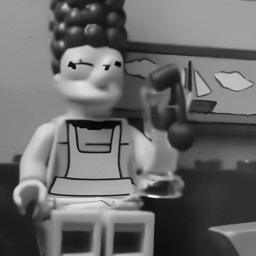}
\\
\hline
& 30.34 / 0.94 & 30.34 / 0.94 & 30.33 / 0.94\\
\CenteredVcell{0.3\textwidth}{ADMM $\sqrt{\vareps}=40/255,\,15/255,\,5/255$} &
\includegraphics[width=0.3\textwidth]{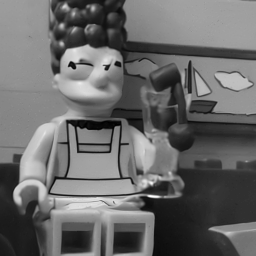} &
\includegraphics[width=0.3\textwidth]{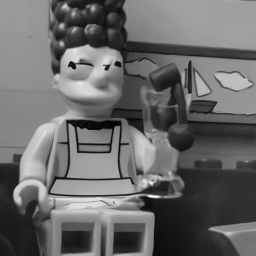} &
\includegraphics[width=0.3\textwidth]{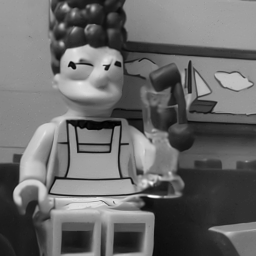}
\end{tabular}
\caption{Inpainting results for the Simpson's image with $p=0.8$, $\sigma=0$ each column corresponds to a different initial condition
}
\label{fig:inpainting-simpsons}

\end{figure}

\begin{acknowledgements}
VDB was partially supported by EPSRC grant EP/R034710/1. RL was partially supported by grants from Région Ile-De-France. AD acknowledges support  of the Lagrange Mathematical and Computing Research Center. MP was partially supported by EPSRC grant EP/T007346/1. JD and AA acknowledge support from the French Research Agency through the PostProdLEAP project (ANR-19-CE23-0027-01).
Computer experiments for this work ran on a Titan Xp GPU donated by NVIDIA, as well as on HPC resources from GENCI-IDRIS (Grant 2020-AD011011641).\end{acknowledgements}

\bibliographystyle{spmpsci}      %
\bibliography{pnp}   %

\end{document}